\documentclass[10pt]{article}
\usepackage{floatrow}

\usepackage{amsfonts,amsmath,amsthm}
\usepackage[numbers]{natbib}
\usepackage{algorithm,algorithmic}
\usepackage{bm}
\usepackage[dvipdfmx]{color}
\usepackage[dvipdfmx]{hyperref}
\usepackage[dvipdfmx]{graphicx}
\usepackage[margin=1in]{geometry}

\usepackage{url}

\usepackage{mcr}
\usepackage{pkg}
\usepackage{algE}
\usepackage{thmE}

\usepackage{subfigure}

\usepackage[numbers]{natbib}
\usepackage{wrapfig}
\usepackage{multirow}
\usepackage{multicol}
\title{Regularization Path of \\ Cross-Validation Error Lower Bounds}

\date{June 22, 2015}

\author{
Atsushi Shibagaki\\
Nagoya Institute of Technology \\
Nagoya, Japan \\
\texttt{shibagaki.a.mllab.nit@gmail.com} \\
\and
Yoshiki Suzuki \\
Nagoya Institute of Technology \\
Nagoya, Japan \\
\texttt{suzuki.mllab.nit@gmail.com} \\
\and
Masayuki Karasuyama \\
Nagoya Institute of Technology \\
Nagoya, Japan \\
\texttt{karasuyama@nitech.ac.jp} \\
\and
Ichiro Takeuchi\thanks{Corresponding author} \\
Nagoya Institute of Technology \\
Nagoya, Japan \\
\texttt{takeuchi.ichiro@nitech.ac.jp} \\
}

\begin{document}

\maketitle

\begin{abstract}
Careful tuning of
a \emph{regularization parameter}
is indispensable
in many machine learning tasks
because it has a significant impact
on generalization performances.
Nevertheless,
current practice of regularization parameter tuning
is more of an art than a science,
e.g.,
it is hard to tell how many grid-points would be needed
in cross-validation (CV)
for obtaining a solution with sufficiently small CV error.
In this paper
we propose a novel framework
for computing
a lower bound of the CV errors
as a function of the regularization parameter,
which we call
\emph{regularization path of CV error lower bounds}.
The proposed framework
can be used
for providing
a theoretical approximation guarantee
on a set of solutions
in the sense that
how far the CV error of the current best solution could be away from best possible CV error
in the entire range of the regularization parameters.
We demonstrate through numerical experiments that
a theoretically guaranteed choice of a regularization parameter in the above sense
is possible
with reasonable computational costs.
%as long as the approximation tolerance
%$\veps$ is not too close to zero.

% All but the simplest machine learning algorithms
% have
% %one or more
% hyper-parameters
% which are usually tuned via cross-validation.
% %
% Despite its significant impact on generalization performances,
% current practice of
% hyper-parameter tuning
% is more of an art than a science
% because
% they are merely selected
% from a heuristically determined finite set of candidates.
% %
% Furthermore,
% if one wants to find sufficiently good hyper-parameter,
% the candidate set must be sufficiently large,
% which makes the cross-validation process
% time-consuming.
% %
% In this paper
% we address these problems
% for a class of regularized convex classifiers
% by introducing a novel hyper-parameter tuning algorithm.
% %
% Using the proposed algorithm
% we can compute
% finite sequence of solutions
% at various hyper-parameter values
% in such a way that
% the best among those finite number of solutions
% is guaranteed to be
% optimal
% or
% $\veps$-suboptimal
% for the entire continuous range of the hyper-parameter.
% %
% In addition,
% cross-validation process can be much more efficient
% because the proposed algorithm is automatically adapted to the ranges of hyper-parameter values
% where the validation errors can be (exactly or approximately) minimized.
% %
% These advantages of the proposed algorithm are also demonstrated through numerical experiments.

\end{abstract}

  \section{Introduction}
 \label{sec:introduction}
 \vspace{-0.5em}
 Many machine learning tasks
 involve careful tuning
 of \emph{a regularization parameter}
 that controls the balance between
 an empirical loss term
 and
 a regularization term.
 A regularization parameter
 is usually selected
 by comparing the cross-validation (CV) errors
 at several different regularization parameters.
 Although
 its choice
 has a significant impact
 on the generalization performances,
 the current practice is still more of an art than a science.
 For example,
 in commonly used grid-search,
 it is hard to tell
 how many grid points we should search over
 for obtaining sufficiently small CV error.

 In this paper
 we introduce a novel framework
 for a class of regularized binary classification problems
 that can compute
 \emph{a regularization path of CV error lower bounds}.
 For an $\veps \in [0, 1]$,
 we define
 \emph{$\veps$-approximate regularization parameters}
 to be a set of regularization parameters
 such that
 the CV error of the solution at the regularization parameter
 is guaranteed to be no greater by
 $\veps$
 than the best possible CV error
 in the entire range of regularization parameters.
 Given a set of solutions
 obtained,
 for example,
 by grid-search,
 the proposed framework allows us to provide
 a theoretical guarantee of the current best solution
 by explicitly quantifying
 its approximation level
 $\veps$
 in the above sense.
 Furthermore,
 when a desired approximation level
 $\veps$
 is specified,
 the proposed framework can be used
 for efficiently finding
 one of the $\veps$-approximate regularization parameters.

 The proposed framework is built on a novel CV error lower bound
 that can be represented as a function of the regularization parameter,
 and
 this is why we call it as
 a regularization path of CV error lower bounds.
 For computing a path,
 no special optimization algorithm is needed.
 We only need to have a finite number of solutions
 obtained by any algorithms.
 It is thus easy to apply our framework
 to common regularization parameter tuning strategies
 such as grid-search or {\it Bayesian optimization}.
 Furthermore,
 the proposed framework can be used
 not only with exact optimal solutions
 but also with sufficiently good approximate solutions,
 which is computationally advantageous
 because
 completely solving an optimization problem is often much more costly
 than obtaining a reasonably good approximate solution.

 Our main contribution in this paper
 is to show that
 a theoretically guaranteed choice of a regularization parameter
 in the above sense
 is possible
 with reasonable computational costs.
 To the best of our knowledge,
 there is no other existing methods
 for providing  such a theoretical guarantee on CV error
 that can be used as generally as ours.
 \figurename~\ref{fig:illust} illustrates the behavior of the algorithm
 for obtaining $\veps = 0.1$ approximate regularization parameter
 (see \S\ref{sec:exp} for the setup).

 \begin{figure}
  \floatbox[{\capbeside\thisfloatsetup{capbesideposition={right,top},capbesidewidth=95mm}}]{figure}[\FBwidth]
  { \hspace*{-7mm}
  \caption{
    An illustration of the proposed framework.
    One of our algorithms presented in \S\ref{sec:algorithm}
 automatically selected
    39
    regularization parameter values
    in
    $[10^{-3}, 10^3]$,
    and
    an upper bound of the validation error
    for each of them
    is obtained
    by solving an optimization problem approximately.
    Among those
    39
    values,
   the one with the smallest validation error upper bound
   (indicated as $\bigstar$ at $C = 1.368$)
   is guaranteed to be
   \emph{$\veps(= 0.1)$ approximate regularization parameter}
   in the sense that
   the validation error for the regularization parameter is no greater
   by $\eps$
   than the smallest possible validation error in the whole interval
   $[10^{-3}, 10^3]$.
   See \S\ref{sec:exp}
   for the setup
   (see also \figurename~\ref{fig:ion_sec6} for the results with other options).
  }\label{fig:illust}}
  {
  \hspace*{-6mm}
  \includegraphics[width=0.35\textwidth]{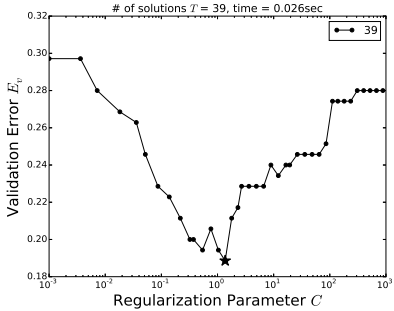}
  % \\
  % {\footnotesize $\veps = 0.1$ with tricks 1 and 2}
  }
 \vspace*{-5mm}
 \end{figure}
 \vspace{-1em}

 % \begin{figure}[t]
 %  \begin{center}
 %   \includegraphics[width=0.35\textwidth]{./fig/2_01_ionosphere_scale}
 %   \\
 %   {\footnotesize $\veps = 0.1$ with tricks 1 and 2}
 %  \end{center}
 %  \vspace*{-2.5mm}
 %  \caption{
 %  An illustration of the algorithm for obtaining
 %  $\veps=0.1$ approximate regularization parameter.
 %  %
 %  The algorithm automatically selected
 %  39
 %  regularization parameters
 %  in
 %  $[10^{-3}, 10^3]$,
 %  and
 %  an upper bound of the CV error
 %  for each of the 39 regularization parameter
 %  is obtained
 %  by solving an optimization problem approximately.
 %  %
 %  Among those
 %  39
 %  values,
 %  the one with the smallest CV error upper bound
 %  (indicated as $\bigstar$ at $C = 1.368$)
 %  is guaranteed to be
 %  \emph{$\veps(= 0.1)$ approximate regularization parameter}
 %  in the sense that
 %  the CV error for the regularization parameter is at most greater
 %  by $\eps$
 %  than the smallest possible CV error in the whole interval
 %  $[10^{-3}, 10^3]$.
 %  See \S\ref{sec:exp}
 %  for the setup
 %  (see also \figurename~\ref{fig:ion_sec6} for the results with other options).
 %  }
 %  \label{fig:illust}
 % \vspace*{-5mm}
 % \end{figure}

\paragraph{Related works}

\emph{Optimal regularization parameter}
can be found
if
its \emph{exact regularization path}
can be computed.
Exact regularization path
has been intensively studied
\cite{Efron2004,Hastie2004},
but
they are known to be numerically unstable
and
do not scale well.
%because
%a large
%matrix inversion is required
%in each step.
%
Furthermore,
exact regularization path
can be computed
only for a limited class of problems
whose solutions are written as piecewise-linear functions of
the regularization parameter
\cite{Rosset2007}.
Our framework is much more efficient
and
can be applied to wider classes of problems
whose exact regularization path cannot be computed.
This work was motivated
by recent studies on
approximate regularization path
\cite{Giesen2012b,Giesen2012a,Giesen2014,Mairal2012}.
These approximate regularization paths
have a property that
the objective function value
at each regularization parameter value
is no greater
by $\veps$
than the optimal objective function value
in the entire range of regularization parameters.
Although
these algorithms are
much more stable and efficient
than
exact ones,
for the task of tuning a regularization parameter,
our interest is
not
in objective function values
but in CV errors.
Our approach is more suitable for regularization parameter tuning tasks
in the sense that
the approximation quality is guaranteed in terms of CV error.

As illustrated in \figurename~\ref{fig:illust},
we only compute
a finite number of solutions,
but still provide
approximation guarantee
in the whole interval of the regularization parameter.
To ensure such a property,
we need to introduce a novel CV error lower bound
that is sufficiently tight and represented as a monotonic function of the regularization parameter.
Although
several CV error bounds
(mostly for leave-one-out CV)
of
SVM and other similar learning frameworks
exist
(e.g., \cite{Vapnik2000,Joachims2000,Chung2003,Lee2004}),
none of them satisfy the above required properties.
The idea of our CV error bound
is inspired from recent studies on
\emph{safe screening}
\cite{ElGhaoui2012,Xiang2011,Ogawa2013,Liu2014,Wang2014}
(see Appendix~\ref{app:proofs} for the detail).
Furthermore,
we emphasize that
our contribution is
\emph{not}
in presenting a new generalization error bound,
but in
introducing
a practical framework for providing a theoretical guarantee
on the choice of a regularization parameter.
Although generalization error bounds
such as structural risk minimization
\cite{Vapnik96}
might be used
for a rough tuning of a regularization parameter,
they are known to be too loose to use as an alternative to CV
(see, e.g., \S11 in \cite{shalev-shwartz2014}).
We also note that
our contribution is \emph{not} in presenting new method for regularization parameter tuning
such as
Bayesian optimization
\cite{Snoek2012},
random search
\cite{Bergstra2012}
and
gradient-based search
\cite{Chapelle2002}.
As we demonstrate in experiments,
our approach can provide a theoretical approximation guarantee
of the regularization parameter selected by these existing methods.

\section{Problem Setup}
\label{sec:problem-setup}
We consider
linear binary classification problems.
Let
$\{(x_i, y_i) \in \RR^d \times \{-1, 1\}\}_{i \in [n]}$
be the training set
where
$n$ is the size of the training set,
$d$ is the input dimension,
and
$[n] := \{1, \ldots, n\}$.
An independent held-out validation set with size $n^\prime$
is denoted similarly as
$\{(x^\prime_i, y^\prime_i) \in \RR^d \times \{-1, 1\}\}_{i \in [n^\prime]}$.
A linear decision function is written as
$f(x) = w^\top x$,
where
$w \in \RR^d$
is a vector of coefficients,
and
$^\top$
represents the transpose.
We assume
the availability of
a held-out validation set
only for simplifying the exposition.
All the proposed methods
presented in this paper
can be straightforwardly adapted to
a cross-validation setup.
Furthermore,
the proposed methods can be
\emph{kernelized}
if the loss function
satisfies a certain condition.
In this paper
we focus
on the following class of
regularized convex loss minimization problems:
\begin{align}
 \label{eq:the-class-of-problems}
 w^*_C :=
 \arg \min_{w \in \RR^d}
 ~
 \frac{1}{2} \|w\|^2 + C \sum_{i \in [n]} \ell(y_i, w^\top x_i),
\end{align}
where
$C > 0$
is the regularization parameter,
and
$\|\cdot\|$
is the Euclidean norm.
%that controls the balance
%between
%the regularization term
%and
%the empirical loss term.
%
The loss function is denoted as
$\ell: \{-1, 1\} \times \RR \to \RR$.
We assume that
$\ell(\cdot, \cdot)$ is convex and subdifferentiable in the 2nd argument.
Examples of such loss functions include
logistic loss,
%$\ell(y_i, w^\top x_i) := \log(1 + \exp(- y_i w^\top x_i))$,
hinge loss,
%$\ell(y_i, w^\top x_i) := \max\{0, (1 - y_i w^\top x_i)^2)$,
Huber-hinge loss,
etc.
%\begin{align}
% \label{eq:huber-hinge}
% \ell(y_i, w^\top x_i)
% :=
% \mycase{
% 1 - y w^\top x
% &
% \text{for }
% y w^\top x < 0,
% \\
% xxx
% &
% \text{for }
% y w^\top x \in [0, 1]
% \\
% 0
% &
% \text{for }
% y w^\top x > 1.
% }
%\end{align}
%
%As a working example,
%we use Huber-hinge loss
%\eq{eq:huber-hinge}
%in the experiments.
%
For notational convenience,
we denote the individual loss as
$\ell_i(w) := \ell(y_i, w^\top x_i)$
for all
$i \in [n]$.
The optimal solution
%of the problem \eq{eq:the-class-of-problems}
for the regularization parameter
$C$
is explicitly denoted as
$w^*_C$.
We assume that the regularization parameter
is defined in a finite interval
$[C_\ell, C_u]$,
e.g.,
$C_\ell = 10^{-3}$
and
$C_u = 10^3$
as we did in the experiments.

For a solution
$w \in \RR^d$,
the validation error\footnote{
For simplicity,
we regard
a validation instance whose score
is exactly zero,
i.e.,
$w^\top x^\prime_i = 0$,
is correctly classified
in
\eq{eq:valid-error}.
Hereafter,
we assume that there are no validation instances whose input vector is completely 0,
i.e.,
$x^\prime_i = 0$,
because those instances are always correctly classified
according to the definition in
\eq{eq:valid-error}.
}
is defined as
 \begin{align}
  \label{eq:valid-error}
  E_v(w)
  :=
  \frac{1}{n^\prime}
  \sum_{i \in [n^\prime]}
  I(y^\prime_i w^\top x^\prime_i < 0),
 \end{align}
where
$I(\cdot)$
is the indicator function.
 %where
%${\rm sign}(\cdot)$
%represents the sign}.
%
In this paper, we consider two problems.
In the first problem,
given a set of (either optimal or approximate) solutions
$w^*_{C_1}, \ldots, w^*_{C_T}$
at
$T$ different regularization parameters
$C_1, \ldots, C_T \in [C_\ell, C_u]$,
we compute the approximation level
$\veps$
such that
\begin{align}
 \label{eq:apprx.quality}
 \min_{C_t \in \{C_1, \ldots, C_T\}}
 \!\!\!\!\!   \!\!\!\!\!
 E_v(w^*_{C_t})
 -
 E_v^*
 \le
\veps,
 ~~~\text{where}~~~
 E_v^*
 :=
  \!\!\!\!\!
 \min_{C \in [C_l, C_u]}
  \!\!\!\!\!
 E_v(w^*_{C}).
\end{align}
In the second problem,
we find
an $\veps$-approximate regularization parameter
within an interval
$C \in [C_l, C_u]$,
which is defined as an element of the following set
 \begin{align*}
  \cC(\veps)
  :=
  \Big\{
  C \in [C_l, C_u]
  ~\Big|~
  E_v(w^*_C) -
  E_v^*
  \le \veps
  \Big\}.
 \end{align*}
Both of these two problems can be solved
by using our proposed framework
for computing
a path of validation error lower bounds.

%\paragraph{Basic Idea}
%%
%In the proposed algorithm,
%a finite sequence of regularization parameter values
%$\{\tilde{C}_{t} \in [C_l, C_u]\}_{t \in [T]}$
%are automatically selected,
%where $T$ itself is also determined by the algorithm.
%%
%For each of these $T$ regularization parameter values,
%the algorithm approximately solves the optimization problem \eq{eq:the-class-of-problems},
%and the approximate solution is used for computing
%a validation error upper bound for each of these $T$ regularization parameter values.
%%
%An essential requirement to the algorithm is that
%the sequence
%$\{\tilde{C}_{t}\}_{t \in [T]}$
%must be selected to satisfy a property that
%the regularization parameter for which the validation error upper bound is smallest among them
%is guaranteed to be
%an $\veps$-approximate regularization parameter.
%%in the interval
%%$[C_l, C_u]$.
%%
%To ensure such a property,
%we should be able to guarantee that
%the validation error is not smaller by
%$\veps$
%than the current smallest validation error upper bound
%for any
%$C \in [\tilde{C}_t, \tilde{C}_{t+1}]$,
%$t \in [T - 1]$.
%%
%To this end,
%we introduce a novel technique
%for computing validation error lower bound
%which is represented
%as a function of the regularization parameter
%$C$.
%%
%In \S\ref{sec:theory}
%we derive the validation error lower bound,
%and
%then present the algorithm
%in \S\ref{sec:algorithm}.

\section{Validation error lower bounds as a function of regularization parameter}
\label{sec:theory}
In this section,
we derive
a validation error lower bound
which is represented
as a function of the regularization parameter
$C$.
Our basic idea
is to compute a lower and an upper bound
of the inner product score
$w^{*\top}_C x^\prime_i$
for each validation input
$x^\prime_i, i \in [n^\prime]$,
as a function of the regularization parameter
$C$.
For computing the bounds of
$w^{*\top}_C x^\prime_i$,
we use a solution
(either optimal or approximate)
for a different regularization parameter
$\tilde{C} \neq C$.

\subsection{Score bounds}
\label{subsec:with-approx}
We first describe how to obtain a lower and an upper bound
of inner product score
$w^{*\top}_C x_i^\prime$
based on an approximate solution
$\hat{w}_{\tilde{C}}$
at a different regularization parameter $\tilde{C} \neq C$.

% -------------------------
% Lemma 3
% -------------------------
\begin{lemm}
 \label{lemm:decision-score-bounds-with-approximate-solutions}
 Let
 $\hat{w}_{\tilde{C}}$
 be an approximate solution
 of the problem
 \eq{eq:the-class-of-problems}
 for a regularization parameter value $\tilde{C}$
 and $\xi_i(\hat{w}_C)$ be a subgradient of $\ell_i$ at
 $w = \hat{w}_C$
 such that
 a subgradient of the objective function is
%  \begin{align}
%   \label{eq:gradient}
%   g(\hat{w}_{\tilde{C}})
%  :=
%  \frac{\partial}{\partial w}
%  \Big(
%  \frac{1}{2} \|w\|^2
%  +
%  \tilde{C}
%  \sum_{i \in [n]}
%  \ell_i(w)
%  \Big)
%  \Big|_{w = \hat{w}_{\tilde{C}}}.
% \end{align}
 \begin{align}
  \label{eq:gradient}
  g(\hat{w}_{\tilde{C}})
 :=
 \hat{w}_C
 +
 \tilde{C}
 \sum_{i \in [n]}
 \xi_i (\hat{w}_C)
 .
\end{align}
 Then,
 for any $C > 0$,
 the score
 $w^{*\top}_C x^\prime_i, i \in [n^\prime],$
 satisfies
 \begin{subequations} %11:11�̎��Q
  \label{eq:DS-approx}
 \begin{align}
  \label{eq:DS-LB-approx}
  w^{*\top}_C x^\prime_i \! \ge \!
  {LB}(w^{*\top}_C x^\prime_i | \hat{w}_{\tilde{C}})
  \! := \!
  \mycase{ \!\!\!
  \phantom{+} \alpha(\hat{w}_{\tilde{C}}, x^\prime_i)
  -
  \frac{1}{\tilde{C}}
  (
  \beta(\hat{w}_{\tilde{C}}, x^\prime_i)
  +
  \gamma(g(\hat{w}_{\tilde{C}}), x^\prime_i)
  ) C,
  % \\
  % \hspace*{20pt}
  {\rm if~}
  C > \tilde{C},
  \\ \! \! \!
  - \beta(\hat{w}_{\tilde{C}}, x^\prime_i)
  +
  \frac{1}{\tilde{C}}
  (
  \alpha(\hat{w}_{\tilde{C}}, x^\prime_i) + \delta(g(\hat{w}_{\tilde{C}}), x^\prime_i)
  ) C,
  % \\
  % \hspace*{20pt}
  {\rm if~}
  C < \tilde{C},
  }
 \end{align}
\begin{align}
  \label{eq:DS-UB-approx}
  w^{*\top}_C x^\prime_i \! \le \!
   {UB}(w^{*\top}_C x^\prime_i | \hat{w}_{\tilde{C}})
  \! := \!
  \mycase{ \! \! \!
  -\beta(\hat{w}_{\tilde{C}}, x^\prime_i)
  +
  \frac{1}{\tilde{C}}
  (
  \alpha(\hat{w}_{\tilde{C}}, x^\prime_i)
  +
  \delta(g(\hat{w}_{\tilde{C}}), x^\prime_i)
  ) C,
  % \\
  % \hspace*{20pt}
  {\rm if~}
  C > \tilde{C},
  \\ \!\!\!
  \phantom{+}\alpha(\hat{w}_{\tilde{C}}, x^\prime_i)
  -
  \frac{1}{\tilde{C}}
  (
  \beta(\hat{w}_{\tilde{C}}, x^\prime_i) + \gamma(g(\hat{w}_{\tilde{C}}), x^\prime_i)
  ) C,
  % \\
  % \hspace*{20pt}
  {\rm if~}
  C < \tilde{C},
  }
 \end{align}
 \end{subequations}
 where
 \begin{align*}
  &
  \alpha(w^*_{\tilde{C}}, x^\prime_i)
  := \frac{1}{2}
  (\|w^*_{\tilde{C}}\| \|x^\prime_i\| + w^{*\top}_{\tilde{C}} x^\prime_i) \ge 0,
  ~
  \gamma(g(\hat{w}_{\tilde{C}}), x^\prime_i)
  := \frac{1}{2}
  (\|g(\hat{w}_{\tilde{C}})\| \|x^\prime_i\| + g(\hat{w}_{\tilde{C}})^\top x^\prime_i) \ge 0,
  \\
  &
  \beta(w^*_{\tilde{C}}, x^\prime_i)
  := \frac{1}{2}
  (\|w^*_{\tilde{C}}\| \|x^\prime_i\| - w^{*\top}_{\tilde{C}} x^\prime_i) \ge 0,
  ~
  \delta(g(\hat{w}_{\tilde{C}}), x^\prime_i)
  := \frac{1}{2}
  (\|g(\hat{w}_{\tilde{C}})\| \|x^\prime_i\| - g(w_{\tilde{C}})^\top x^\prime_i) \ge 0.
 \end{align*}
\end{lemm}
The proof is presented in
Appendix~\ref{app:proofs}.
Lemma
\ref{lemm:decision-score-bounds-with-approximate-solutions}
tells that
we have a lower and an upper bound of the score
$w^{*\top}_C x^\prime_i$
for each validation instance
that linearly change with
the regularization parameter $C$.
When
$\hat{w}_{\tilde{C}}$
is optimal,
it can be shown that
(see Proposition B.24 in \cite{Bertsekas1999})
there exists a subgradient such that
$g(\hat{w}_{\tilde{C}}) =0$,
meaning that
the bounds are tight
because
$\gamma(g(\hat{w}_{\tilde{C}}), x_i^\prime)$
$=$
$\delta(g(\hat{w}_{\tilde{C}}), x_i^\prime)$
$=0$.

\begin{coro}
 \label{coro:DSB-property1}
 When
 $C = \tilde{C}$,
 the score
 $w^{*\top}_{\tilde{C}} x^\prime_i, i \in [n^\prime],$
 for the regularization parameter value
 $\tilde{C}$
 itself
 satisfies
 \begin{align*}
  w^{*\top}_{\tilde{C}} \! x^\prime_i \! \ge \!\!
  {LB}(w^{*\top}_{\tilde{C}} \! x^\prime_i | \hat{w}_{\tilde{C}})
  \! = \!
  \hat{w}^{\top}_{\tilde{C}} x^\prime_i
  \! - \!
  \gamma(g(\hat{w}_{\tilde{C}}), x^\prime_i), ~
  w^{*\top}_{\tilde{C}} \! x^\prime_i \! \le \!
  {UB}(w^{*\top}_{\tilde{C}} x^\prime_i | \hat{w}_{\tilde{C}})
  \! = \!
  \hat{w}^{\top}_{\tilde{C}}  x^\prime_i
  \! + \!
  \delta(g(\hat{w}_{\tilde{C}}), x^\prime_i).
 \end{align*}
\end{coro}
The results in Corollary
\ref{coro:DSB-property1}
are obtained by simply substituting $C = \tilde{C}$ into
\eq{eq:DS-LB-approx}
and
\eq{eq:DS-UB-approx}.
%
%The bounds in
%Corollary
%\ref{coro:DSB-property1}
%can be used
%as a stopping criteria of an optimization algorithm
%when we compute an approximate solution
%$\hat{w}_{\tilde{C}}$.

\subsection{Validation Error Bounds}
\label{subs:validation-error-bounds}
Given a lower and an upper bound of the score of each validation instance,
a lower bound of the validation error can be computed
by simply using the following facts:
\begin{subequations} %08:21�̎��Q
 \label{eq:mis-classified}
\begin{align}
 \label{eq:mis-classified.a}
 \hspace*{-2.5mm}
 y^\prime_i = +1
 \text{ and }
 {UB}(w^{*\top}_C x^\prime_i | \hat{w}_{\tilde{C}}) < 0
 &\Rightarrow
 \text{mis-classified},
 \\
 \label{eq:mis-classified.b}
 \hspace*{-2.5mm}
 y^\prime_i = -1
 \text{ and }
 {LB}(w^{*\top}_C x^\prime_i | \hat{w}_{\tilde{C}}) > 0
 &\Rightarrow
 \text{mis-classified}.
\end{align}
\end{subequations}
Furthermore,
since the bounds in
Lemma
\ref{lemm:decision-score-bounds-with-approximate-solutions}
linearly change
with the regularization parameter
$C$,
we can identify
the interval of
$C$
within which
the validation instance is guaranteed to be
mis-classified.
\begin{lemm}
 \label{lemm:mis-classified-range}
 For a validation instance with
 $y^\prime_i = +1$,
 if
 \begin{align*}
  \tilde{C}
  % \le
  <
%  C
%  \text{ and }
  C
  <
  \frac{
  \beta(\hat{w}_{\tilde{C}}, x^\prime_i)
  }{
  \alpha(\hat{w}_{\tilde{C}}, x^\prime_i) + \delta(g(\hat{w}_{\tilde{C}}), x^\prime_i)
  } \tilde{C}
  ~~~
  {\rm or}
  ~~~
  \frac{
  \alpha(\hat{w}_{\tilde{C}}, x^\prime_i)
  }{
  \beta(\hat{w}_{\tilde{C}}, x^\prime_i) + \gamma(g(\hat{w}_{\tilde{C}}), x^\prime_i)
  } \tilde{C}
  <
%  C
%  \text{ and }
  C
  % \le
  <
  \tilde{C},
 \end{align*}
 then
 the validation instance
 $(x^\prime_i, y^\prime_i)$
 is mis-classified.
 Similarly,
 for a validation instance with
 $y^\prime_i = -1$,
 if
 \begin{align*}
  \tilde{C}
  % \le
  <
%  C
%  \text{ and }
  C
  <
  \frac{
  \alpha(\hat{w}_{\tilde{C}}, x^\prime_i)
  }{
  \beta(\hat{w}_{\tilde{C}}, x^\prime_i) + \gamma(g(\hat{w}_{\tilde{C}}), x^\prime_i)
  } \tilde{C}
  ~~~
  {\rm or}
  ~~~
  \frac{
  \beta(\hat{w}_{\tilde{C}}, x^\prime_i)
  }{
  \alpha(\hat{w}_{\tilde{C}}, x^\prime_i) + \delta(g(\hat{w}_{\tilde{C}}), x^\prime_i)
  } \tilde{C}
  <
%  C
%  \text{ and }
  C
  % \le
  <
  \tilde{C},
 \end{align*}
 then
 the validation instance
 $(x^\prime_i, y^\prime_i)$
 is mis-classified.
\end{lemm}
This lemma can be easily shown by
applying
\eq{eq:DS-approx}
to
\eq{eq:mis-classified}.

Using
Lemma~\ref{lemm:mis-classified-range},
the lower bound of the validation error is
represented
as a function of the regularization parameter
$C$
in the following form.
\begin{theo}
 \label{theo:valid-err-LB}
 Using an approximate solution
 $\hat{w}_{\tilde{C}}$
 for a regularization parameter
 $\tilde{C}$,
 the validation error
 $E_v(w^*_C)$
 for any
 $C > 0$
 satisfies
 \begin{align}
  \label{eq:valid-err-LB}
  &
  E_v(w^*_C) \ge
  {LB}(E_v(w^*_C) | \hat{w}_{\tilde{C}})
  :=
  \\
  \nonumber
  &
  % :=
  \!\!\!\!\!
  \frac{1}{n^\prime}
  \Bigg( \!
  \sum_{y^\prime_i = +1} \!\!
  I \! \bigg( \!
  \tilde{C} \!\!
  % \le
  <
  \!\!
  C \!\!
  < \!\!
  \frac{
  \beta(\hat{w}_{\tilde{C}}, x^\prime_i)
  }{
  \alpha(\hat{w}_{\tilde{C}}, x^\prime_i) \! + \! \delta(g(\hat{w}_{\tilde{C}}), x^\prime_i)
  } \tilde{C} \!
  \bigg)
  % \\
  % \nonumber
  % &
  % ~~~~~ ~~~
  \! \! + \!\!
  \sum_{y^\prime_i = +1}
  \! \! I \!
  \bigg( \!
  \frac{
  \alpha(\hat{w}_{\tilde{C}}, x^\prime_i)
  }{
  \beta(\hat{w}_{\tilde{C}}, x^\prime_i) \! + \! \gamma(g(\hat{w}_{\tilde{C}}), x^\prime_i)
  } \tilde{C}
  \! \! < \! \!
  C
  \! \!
  % \le
  <
   \!\!
  \tilde{C}
  \bigg)
  \\
  \nonumber
  &
  %~~~~~ ~~~
  + \!\!
  \sum_{y^\prime_i = -1} \!\!
  I
  \bigg( \!
  \tilde{C}
  \!\!
  % \le
  <
  \!\!
  C
  \!\! < \!\!
  \frac{
  \alpha(\hat{w}_{\tilde{C}}, x^\prime_i)
  }{
  \beta(\hat{w}_{\tilde{C}}, x^\prime_i) \! + \! \gamma(g(\hat{w}_{\tilde{C}}), x^\prime_i)
  } \tilde{C} \!
  \bigg)
  % \\
  % \nonumber
  % &
  % ~~~~~ ~~~
  \! \! + \!\!
  \sum_{y^\prime_i = -1} \!\!
  I
  \bigg( \!
  \frac{
  \beta(\hat{w}_{\tilde{C}}, x^\prime_i)
  }{
  \alpha(\hat{w}_{\tilde{C}}, x^\prime_i) \! + \! \delta(g(\hat{w}_{\tilde{C}}), x^\prime_i)
  } \tilde{C}
  \!\! < \!\!
  C
  \!\!
  % \le
  <
   \!\!
  \tilde{C} \!
  \bigg)
  \Bigg).
 \end{align}
\end{theo}
Theorem
\ref{theo:valid-err-LB}
is a direct consequence of
Lemma
\ref{lemm:mis-classified-range}.
The lower bound
\eq{eq:valid-err-LB}
is a staircase function of the regularization parameter
$C$.
%as illustrated in
%\figurename~\ref{fig:ion_lbve}.

By setting
$C = \tilde{C}$,
we can obtain a lower and an upper bound of the validation error
for the regularization parameter
$\tilde{C}$
itself,
which are used in the algorithm as a stopping criteria
for obtaining an approximate solution
$\hat{w}_{\tilde{C}}$.
%
% --------------------------------------------------
% Corollary 7
% --------------------------------------------------
\begin{coro}
 Given an approximate solution
 $\hat{w}_{\tilde{C}}$,
 the validation error
 $E_v(w^*_{\tilde{C}})$
 satisfies
 \begin{subequations}
  \label{eq:valid-err-tilde}
 \begin{align}
  \nonumber
  &
E_v(w^*_{\tilde{C}}) \ge
  {LB}(E_v(w^*_{\tilde{C}}) | \hat{w}_{\tilde{C}})
  \\
  % \nonumber
  &=
  \frac{1}{n^\prime}
  \Bigg(
  \sum_{y^\prime_i = +1}
  I \big(
  \hat{w}^{\top}_{\tilde{C}} x^\prime_i
  +
  \delta(g(\hat{w}_{\tilde{C}}), x^\prime_i)
  < 0
  \big)
  % \\
  % &
  % ~~~~~ ~~~~
  +
  \sum_{y^\prime_i = -1}
  I \big(
  \hat{w}^{\top}_{\tilde{C}} x^\prime_i
  -
  \gamma(g(\hat{w}_{\tilde{C}}), x^\prime_i)
  > 0
  \big)
  \Bigg),
  \\
 %\end{align}
 %%
 %\begin{align}
  \nonumber
  &E_v(w^*_{\tilde{C}}) \le
  {UB}(E_v(w^*_{\tilde{C}}) | \hat{w}_{\tilde{C}})
  \\
  % \nonumber
  &=
  1 -
  \frac{1}{n^\prime}
  \Bigg(
  \sum_{y^\prime_i = +1}
  I \big(
  \hat{w}^{\top}_{\tilde{C}} x^\prime_i
  -
  \gamma(g(\hat{w}_{\tilde{C}}), x^\prime_i)
  \ge 0
  \big)
  % \\
  % &
  % ~~~~~ ~~~~
  \label{eq:valid-err-UB-tilde}
   +
  \sum_{y^\prime_i = -1}
  I \big(
  \hat{w}^{\top}_{\tilde{C}} x^\prime_i
  +
  \delta(g(\hat{w}_{\tilde{C}}), x^\prime_i)
  \le 0
  \big)
  \Bigg).
  \end{align}
 \end{subequations}

\end{coro}

\begin{algorithm}[t]
 \caption{Computing the approximation level $\veps$ from the given set of solutions~~~~~~~~~~~~~~~~~~~~~~~~~~~~~~}
 \label{alg:compute.eps}
 \begin{algorithmic}[1]
   \REQUIRE
   $\{(x_i, y_i)\}_{i \in [n]}$,
   $\{(x^\prime_i, y^\prime_i)\}_{i \in [n^\prime]}$,
   $C_l$,
   $C_u$,
   % $\hat{\cal{W}} := \{ \hat{w}_{C_1} , \ldots, \hat{w}_{C_N} \}$
   $\cW := \{ w_{\tilde{C}_1} , \ldots, w_{\tilde{C}_T} \}$
   \STATE $E_v^{\rm best} \lA \min_{ \tilde{C}_{t} \in \{ \tilde{C}_1, \ldots, \tilde{C}_T \}  } {UB} (E_v(w^*_{\tilde{C}_t}) | {w}_{\tilde{C}_{t}} )  $
   \STATE
        % \begin{align*}
        $
          LB(E_v^*) \lA \min_{c \in [C_l, C_u]} \big\{
          % & {LB}(E_v(C)), \\
          % {\rm where } ~~~~
          % &
          % {LB}(E_v(C))
          % :=
          % \max_{ \tilde{C}_{t} \in \{ \tilde{C}_1, \ldots, \tilde{C}_N \} }  {LB} (E_v(w^*_C) | \hat{w}_{C_{t}} )
          \max_{ \tilde{C}_{t} \in \{ \tilde{C}_1, \ldots, \tilde{C}_T \} }  {LB} (E_v(w^*_c) | {w}_{\tilde{C}_t} )
          \big\}
        $
        % \end{align*}

   \ENSURE
   $\veps = E_v^{\rm best}  - LB(E_v^*)$
 \end{algorithmic}
\end{algorithm}

%\begin{figure}[t]
% \begin{center}
%  \includegraphics[width=0.45\textwidth]{./fig/ion_lbve_13_16}
%  \caption{An illustrative example of the algorithm behavior.
%  %
%  The blue real lines represent the validation error lower bound.
%  %
%  The red chained lines and green dashed lines indicate
%  the current best validation error upper bound
%  $E^{\rm best}_v$
%  and
%  $E^{\rm best}_v - \veps$,
%  respectively.
%  %
%  If the blue validation error lower bound falls below the green ones,
%  %it indicates that
%  the validation error can be smaller
%  by
%  $\eps$
%  than the current best.
%  %
%  In such a case,
%  the algorithm computes the next approximate solution,
%  and update the validation error lower bound
%  based on the new approximate solution.
%  %
%  The plot is an enlarged view of
%  the shaded region
%  (from $\tilde{C}_{13}$ to $\tilde{C}_{17}$)
%  in
%  \figurename~\ref{fig:ion_sec6} (a)
%  in \S\ref{sec:exp}.
%  }
%  \label{fig:ion_lbve}
% \end{center}
%\end{figure}
--------------------------------------------------
algorithm
--------------------------------------------------
\begin{algorithm}[t]
 \caption{Finding an $\veps$ approximate regularization parameter with approximate solutions}
 \label{alg:find.apprx}
 \begin{algorithmic}[1]
   \REQUIRE
   $\{(x_i, y_i)\}_{i \in [n]}$,
   $\{(x^\prime_i, y^\prime_i)\}_{i \in [n^\prime]}$,
   $C_l$,
   $C_u$,
   $\veps$
   \STATE $t \lA 1$, $\tilde{C}_t \lA C_l$, $C^{\rm best} \lA C_l$, $E_v^{\rm best} \lA 1$
   \WHILE{$\tilde{C}_t \le C_u$}
   \STATE $\hat{w}_{\tilde{C}_t}$ $\lA$ solve \eq{eq:the-class-of-problems} approximately for $C = \tilde{C}_t$
   \STATE Compute ${UB}(E_v(w^*_{\tilde{C}_t}) | \hat{w}_{\tilde{C}_t})$ by \eq{eq:valid-err-UB-tilde}.
   \IF{${UB}(E_v(w^*_{\tilde{C}_t}) | \hat{w}_{\tilde{C}_t}) < E_v^{\rm best} $}
   \STATE $E_v^{\rm best} \lA {UB}(E_v(w^*_{\tilde{C}_t}) | \hat{w}_{\tilde{C}_t})$, $C^{\rm best} \lA \tilde{C}_t$
   \ENDIF
   \STATE Set $\tilde{C}_{t+1}$ by \eq{eq:nextC}
   \STATE $t \lA t+1$
   \ENDWHILE
   \ENSURE
  $C^{\rm best} \in \cC(\veps)$.
  %$C^{\rm best}$.
   %an $\veps$-approximate regularization parameter $C^{\rm best}$;
 \end{algorithmic}
\end{algorithm}

\section{Algorithm}
\label{sec:algorithm}
In this section
we present two algorithms
for each of the two problems
discussed in
\S\ref{sec:problem-setup}.
Due to the space limitation,
we roughly describe the most fundamental forms of these algorithms.
Details and several extensions of the algorithms are presented
in supplementary appendices~\ref{app:extensions} , \ref{app:path} and \ref{app:cv}.

\subsection{Problem 1: Computing the approximation level $\veps$ from a given set of solutions}
Given a set of
(either optimal or approximate)
solutions
$\hat{w}_{\tilde{C}_1}, \ldots, \hat{w}_{\tilde{C}_T}$,
obtained
e.g.,
by ordinary grid-search,
our first problem is to
provide a theoretical approximation level
$\veps$
in the sense of \eq{eq:apprx.quality}\footnote{
When we only have approximate solutions
$\hat{w}_{\tilde{C}_1}, \ldots, \hat{w}_{\tilde{C}_T}$,
Eq. \eq{eq:apprx.quality}
is slightly incorrect.
The first term of the l.h.s. of
\eq{eq:apprx.quality}
should be
$\min_{\tilde{C}_t \in \{\tilde{C}_1, \ldots, \tilde{C}_T\}} UB(E_v(\hat{w}_{\tilde{C}_t}) | \hat{w}_{\tilde{C}_t})$.
}.
This problem can be solved easily
by using the validation error lower bounds
developed in \S\ref{subs:validation-error-bounds}.
The algorithm is presented in
Algorithm~\ref{alg:compute.eps},
where
we compute the current best validation error
$E_v^{\rm best}$
in line 1,
and
a lower bound of the best possible validation error
$E_v^* := \min_{C \in [C_\ell, C_u]} E_v(w^*_C)$
in line 2.
Then,
the approximation level
$\veps$
can be simply obtained
by subtracting the latter from the former.
We note that
$LB(E_v^*)$,
the lower bound of $E_v^*$,
can be easily computed
by using $T$ valuation error lower bounds
${LB}(E_v(w_C^*) | w_{\tilde{C}_t}),$
$t = 1, \ldots, T$,
because they are represented as staircase functions of
$C$.

\subsection{Problem 2: Finding an $\veps$-approximate regularization parameter}
\label{subs:find.apprx}
Given a desired approximation level
$\veps$
such as
$\veps = 0.01$,
our second problem is to find an $\veps$-approximate regularization parameter.
To this end
we develop an algorithm
that produces a set of optimal or approximate soluitons
$\hat{w}_{\tilde{C_1}}, \ldots, \hat{w}_{\tilde{C_T}}$
such that,
if we apply Algorithm~\ref{alg:compute.eps} to this sequence,
then approximation level would be smaller than or equal to $\veps$.
%
% Algorithm~\ref{alg:find.apprx}
Algorithm 2
is the pseudo-code of this algorithm.
It computes approximate solutions
for an increasing sequence of regularization parameters
in the main loop (lines 2-11).

% \begin{wrapfigure}{l}{67mm}
%   % \vspace{3em}
%   \hspace{-12mm}
%   \includegraphics[width=55mm]{eps_algo2}
%     % \label{alg:find.apprx}
%     \vspace{-1\baselineskip}
% \end{wrapfigure}

Let us now
consider
$t^{\rm th}$
iteration in the main loop,
where
we have already computed
$t-1$
approximate solutions
$\hat{w}_{\tilde{C}_1}, \ldots, \hat{w}_{\tilde{C}_{t-1}}$
for
$\tilde{C}_1 < \ldots < \tilde{C}_{t-1}$.
At this point,
%\begin{align*}
$$
C^{\rm best} := \arg \!\! \min_{ \tilde{C}_\tau \in \{\tilde{C}_1, \ldots, \tilde{C}_{t - 1}\}}  \!\!
 {UB}(E_v(w^*_{\tilde{C}_\tau}) | \hat{w}_{\tilde{C}_\tau}),
$$
%\end{align*}
is the best (in worst-case) regularization parameter obtained so far and it is guaranteed to be
an $\veps$-approximate regularization parameter
in the interval
$\![C_l, \! \tilde{C}_{t}]\!$
in the sense that the validation error,
%\begin{align*}
$$
E_v^{\rm best} := \min_{\tilde{C}_\tau \in \{\tilde{C}_1, \ldots, \tilde{C}_{t - 1}\}} {UB}(E_v(w^*_{\tilde{C}_\tau}) | \hat{w}_{\tilde{C}_\tau}),
$$
%\end{align*}
is shown to be at most greater
by $\veps$
than the smallest possible validation error in the interval
$[C_l, \tilde{C}_{t}]$.
However,
we are not sure whether
$C^{\rm best}$
can still keep $\veps$-approximation property for
$C > \tilde{C}_t$.
Thus,
in line 3,
we approximately solve the optimization problem
\eq{eq:the-class-of-problems}
at $C = \tilde{C}_t$
and obtain an approximate solution
$\hat{w}_{\tilde{C}_t}$.
Note that
the approximate solution
$\hat{w}_{\tilde{C}_t}$
must be sufficiently good enough
in the sense that
%\begin{align*}
$
{UB}(E_v(w^*_{\tilde{C}_\tau}) | \hat{w}_{\tilde{C}_\tau})
 -
{LB}(E_v(w^*_{\tilde{C}_\tau}) | \hat{w}_{\tilde{C}_\tau})
$
%\end{align*}
is sufficiently smaller than $\veps$
(typically 0.1$\veps$).
If the upper bound of the validation error
${UB}(E_v(w^*_{\tilde{C}_\tau}) | \hat{w}_{\tilde{C}_\tau})$
is smaller than $E_v^{\rm best}$,
we update
$E_v^{\rm best}$
and
$C^{\rm best}$
(lines 5-8).

Our next task is to find
$\tilde{C}_{t+1}$
in such a way that
$C^{\rm best}$
is an $\veps$-approximate regularization parameter in the interval
$[C_l, \tilde{C}_{t+1}]$.
%
%Here,
%we use the validation error lower bound in
%Theorem~\ref{theo:valid-err-LB},
%which is represented
%as a monotonically decreasing staircase function of the regularization parameter $C$
%for $C > \tilde{C}_t$.
%
Using the validation error lower bound
in Theorem~\ref{theo:valid-err-LB},
the task is to find
the smallest
$\tilde{C}_{t+1} > \tilde{C}_{t}$
that violates
\begin{align}
 \label{eq:condition-tp1}
 E_v^{\rm best}
 -
 {LB}(E_v(w^*_C) | \hat{w}_{\tilde{C}_t})
 \le \veps,
 ~~~
 \forall
 C \in [\tilde{C}_t, {C}_{u}],
 % C \in [\tilde{C}_t, \tilde{C}_{t+1}].
\end{align}
In order to formulate such a $\tilde{C}_{t+1}$,
let us define
\begin{align*}
 \cP
 &
 := \{i \in [n^\prime] | y^\prime_i = +1, {UB}(w^{*\top}_{\tilde{C}_t} x^\prime_i | \hat{w}_{\tilde{C}_t}) < 0\},
 % \\
 \cN
 % &
 := \{i \in [n^\prime] | y^\prime_i = -1, {LB}(w^{*\top}_{\tilde{C}_t} x^\prime_i | \hat{w}_{\tilde{C}_t}) > 0\}.
\end{align*}
Furthermore,
let
\begin{align*}
 \Gamma
 :=
% \bigg\{
 % &
 \Big\{
 \frac{
 \beta(\hat{w}_{\tilde{C}_t}, x^\prime_i)
 }{
 \alpha(\hat{w}_{\tilde{C}_t}, x^\prime_i) + \delta( g(\hat{w}_{\tilde{C}_t}), x^\prime_i)
 } \tilde{C}_t
 \Big\}_{i \in \cP}
 % \\
 \cup
 % &
 \Big\{
 \frac{
 \alpha(\hat{w}_{\tilde{C}_t}, x^\prime_i)
 }{
 \beta(\hat{w}_{\tilde{C}_t}, x^\prime_i) + \gamma(g(\hat{w}_{\tilde{C}_t}), x^\prime_i)
 } \tilde{C}_t
 \Big\}_{i \in \cN},
% \bigg\}
% \cap
% \bigg[\tilde{C}_t, C_u\bigg]
\end{align*}
and
denote the $k^{\rm th}$-smallest element of
$\Gamma$
as
$k^{\rm th}(\Gamma)$
for any natural number $k$.
Then,
the smallest
$\tilde{C}_{t+1} > \tilde{C}_{t}$
that violates
\eq{eq:condition-tp1}
is given as
\begin{align}
 \label{eq:nextC}
 \hspace*{-2mm}
 \tilde{C}_{t+1}
 \! \lA \!
 (\lfloor
 n^\prime
 ({LB}( E_v(w^*_{\tilde{C}_t}) | \hat{w}_{\tilde{C}_t})
 \! - \!
 E_v^{\rm best}
 \!
 +
 \!
 \veps)
 \rfloor
 \!+\!
 1)^{\rm th}(\Gamma).
\end{align}
%
%\figurename~\ref{fig:ion_lbve}
%depicts how
%$\tilde{C}_{t+1}$
%is determined.

%\paragraph{Properties of the algorithm}

%\begin{algorithm}[t]
%
% \caption{({\bf $\veps$-reg-opt}): Finding $\veps$-approximate regularization parameter with optimal solutions}
%
% \label{alg:veps-reg-opt}
%
% \begin{algorithmic}[1]
%
%   \REQUIRE
%   $\{(x_i, y_i)\}_{i \in [n]}$,
%   $\{(x^\prime_i, y^\prime_i)\}_{i \in [n^\prime]}$,
%   $C_l$,
%   $C_u$,
%   $\veps$
%
%   \STATE $t \lA 1$, $\tilde{C}_t \lA C_l$, $C^{\rm best} \lA C_l$, $E_v^{\rm best} \lA 1$
%
%   \WHILE{$\tilde{C}_t \le C_u$}
%
%   \STATE Solve \eq{eq:the-class-of-problems} at $C = \tilde{C}_t$ and obtain $w^*_{\tilde{C}_t}$
%
%   \STATE Compute $E_v(w^*_{\tilde{C}_t})$
%
%   \IF{$E_v(w^*_{\tilde{C}_t}) < E_v^{\rm best} $}
%
%   \STATE $E_v^{\rm best} \lA E_v(w^*_{\tilde{C}_t})$, $C^{\rm best} \lA \tilde{C}_t$
%
%   \ENDIF
%
%   \STATE Set $\tilde{C}_{t+1}$ by \eq{eq:}
%
%   \STATE $t \lA t+1$
%
%   \ENDWHILE
%
%   \ENSURE
%   an $\veps$-approximate regularization parameter $C^{\rm best}$;
%
% \end{algorithmic}
%
%\end{algorithm}

%\input{sec5}

	\begin{figure*}[t]
  \begin{center}
   \begin{tabular}{ccc}
    \includegraphics[width=0.35\textwidth]{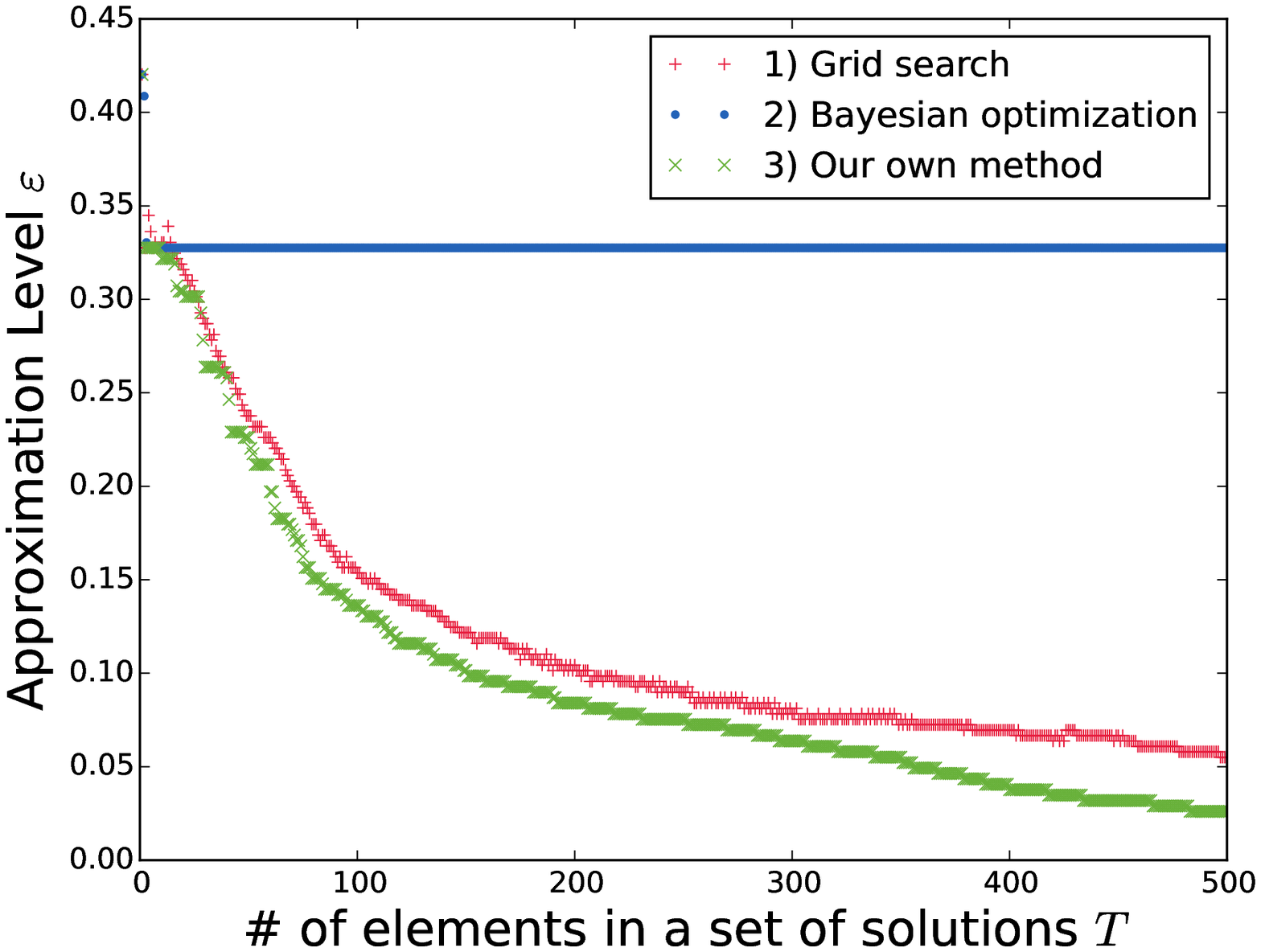}
    \hspace*{-5mm} & \hspace*{-5mm}
    \includegraphics[width=0.35\textwidth]{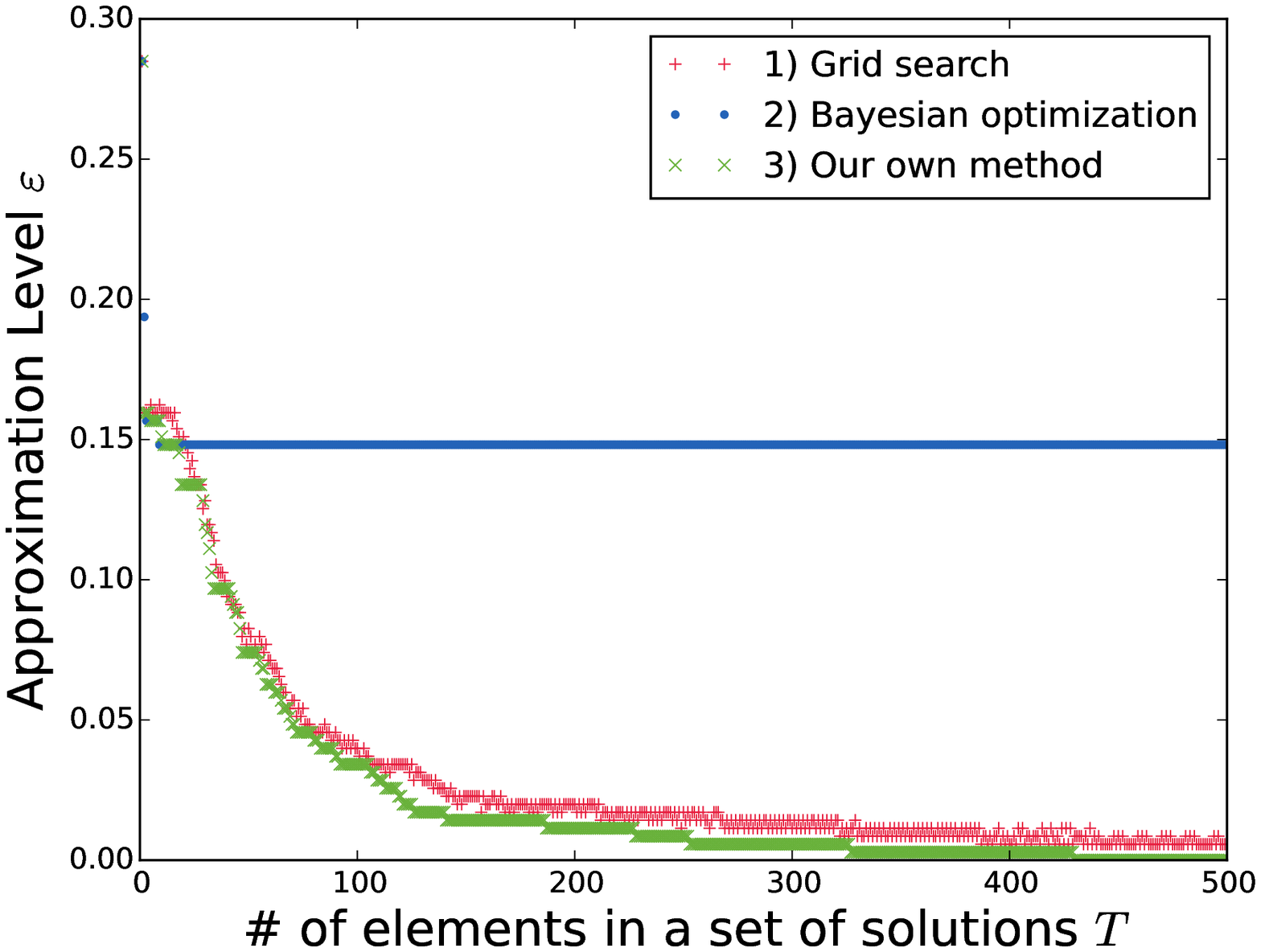}
    \hspace*{-5mm} & \hspace*{-5mm}
    \includegraphics[width=0.35\textwidth]{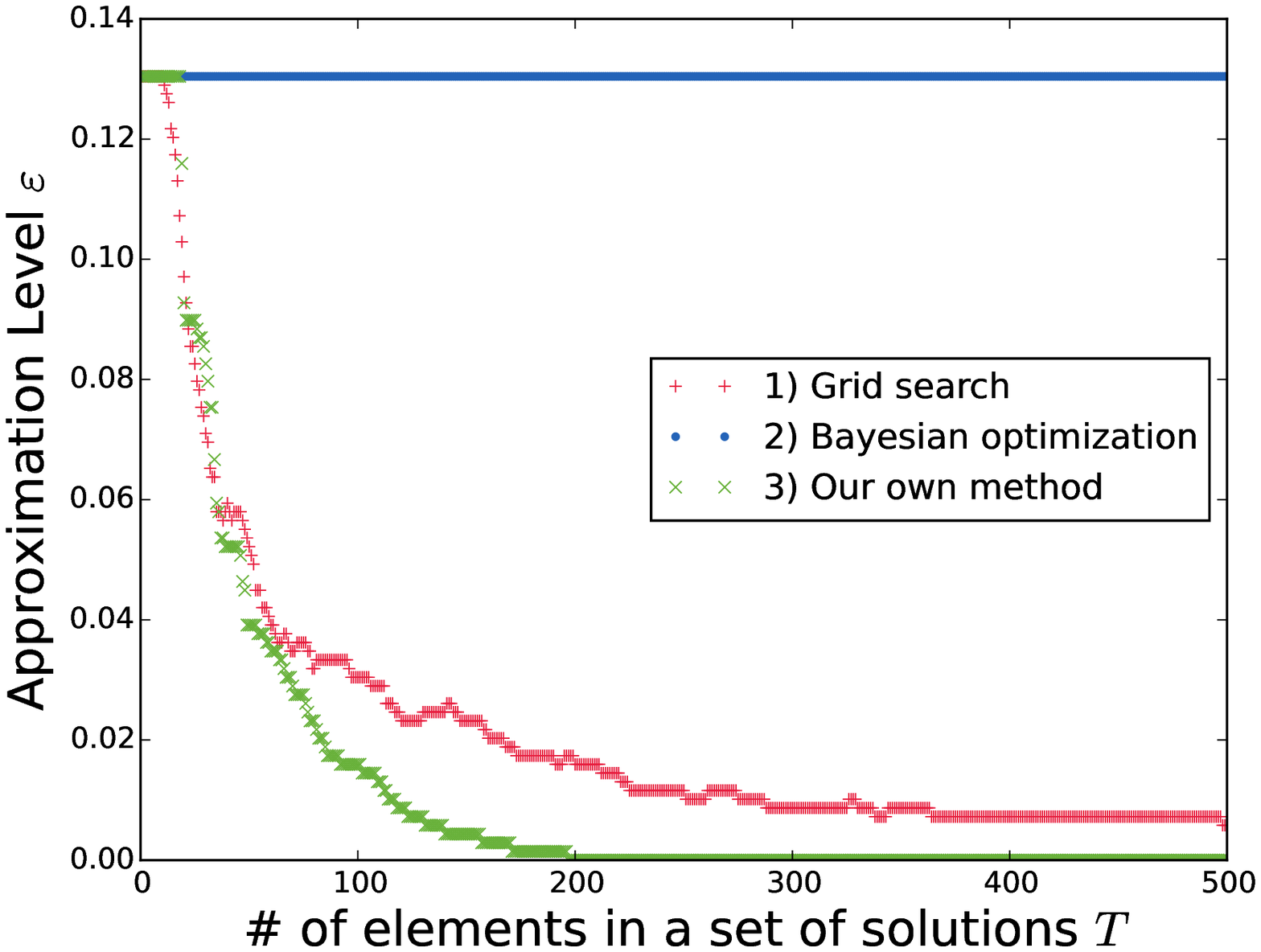} \\
    {\footnotesize ~~~{\tt liver-disorders } (D2) }
    &
    {\footnotesize {\tt ionosphere } (D3) }
    &
    {\footnotesize {\tt australian } (D4) }
   \end{tabular}
 \end{center}
 \caption{
  Illustrations of Algorithm~\ref{alg:compute.eps}
  on three benchmark datasets (D2, D3, D4).
  The plots indicate how the approximation level
  $\veps$
  improves
  as the number of solutions $T$
  increases
  in grid-search (red),
  Bayesian optimization (blue)
  and
  our own method (green, see the main text).
 }
 \label{fig:T_sec6}
	\end{figure*}

 \begin{figure*}[t]
  \begin{center}
   \begin{tabular}{ccc}
    \includegraphics[width=0.35\textwidth]{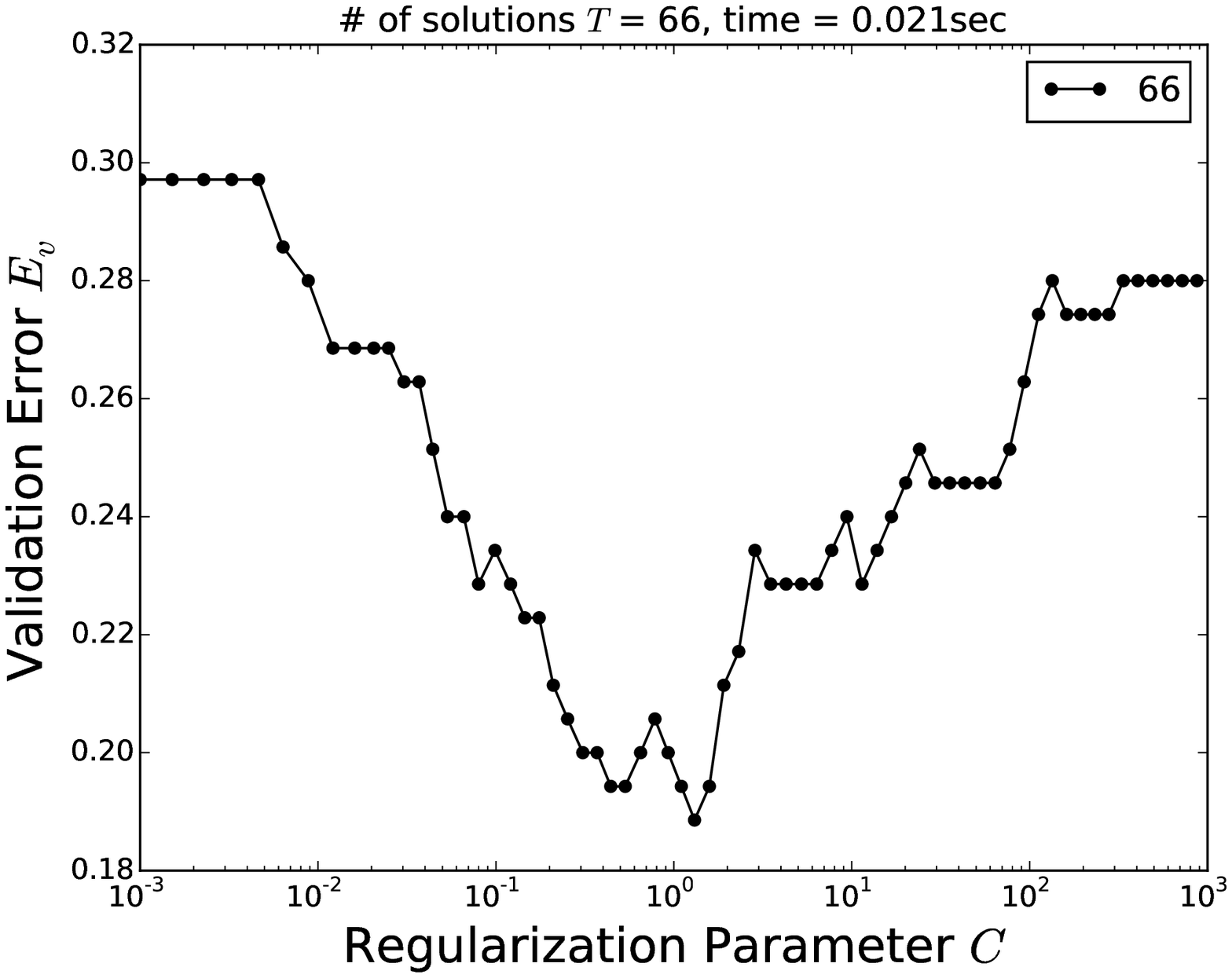}
    \hspace*{-5mm} & \hspace*{-5mm}
    \includegraphics[width=0.35\textwidth]{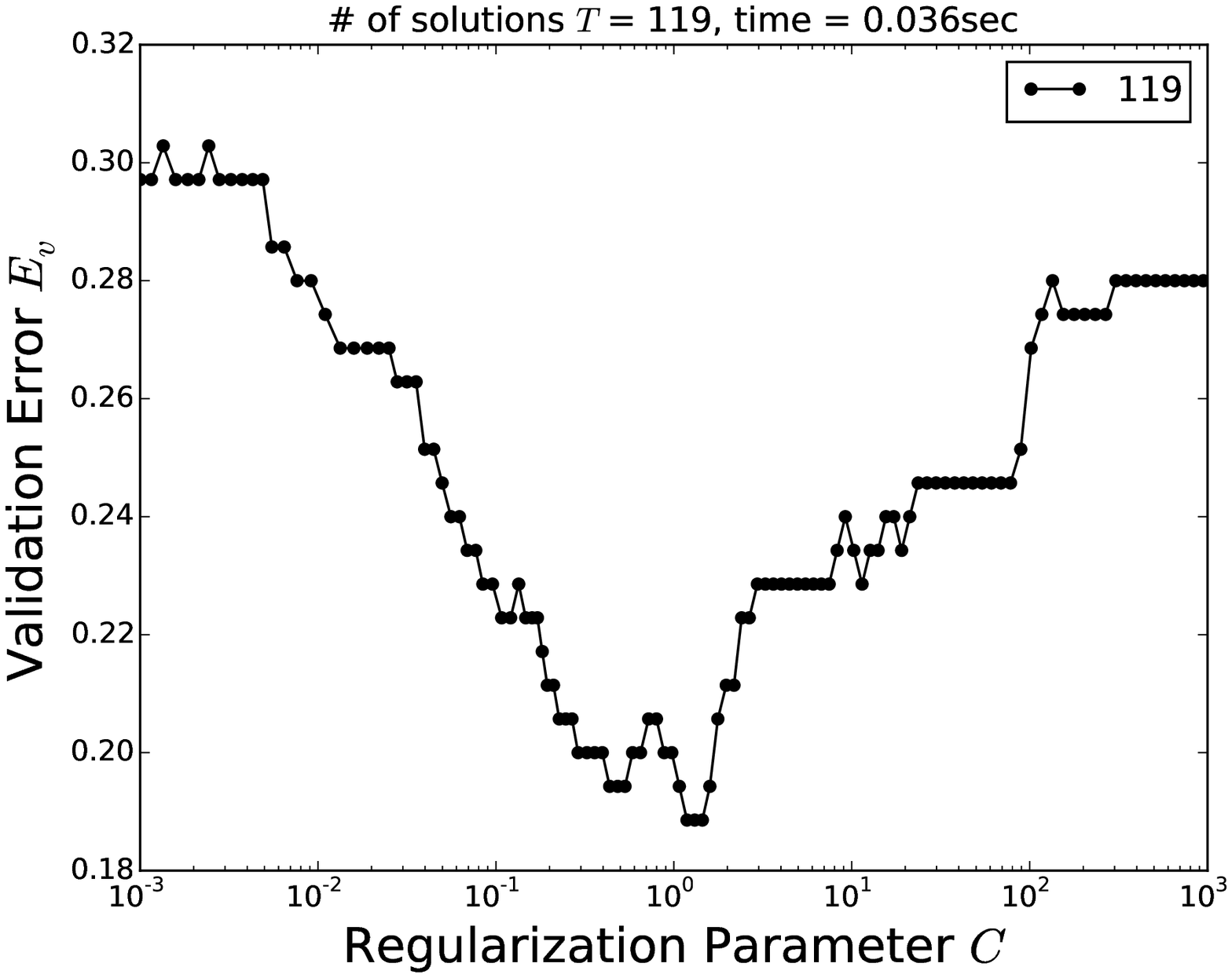}
    \hspace*{-5mm} & \hspace*{-5mm}
    \includegraphics[width=0.35\textwidth]{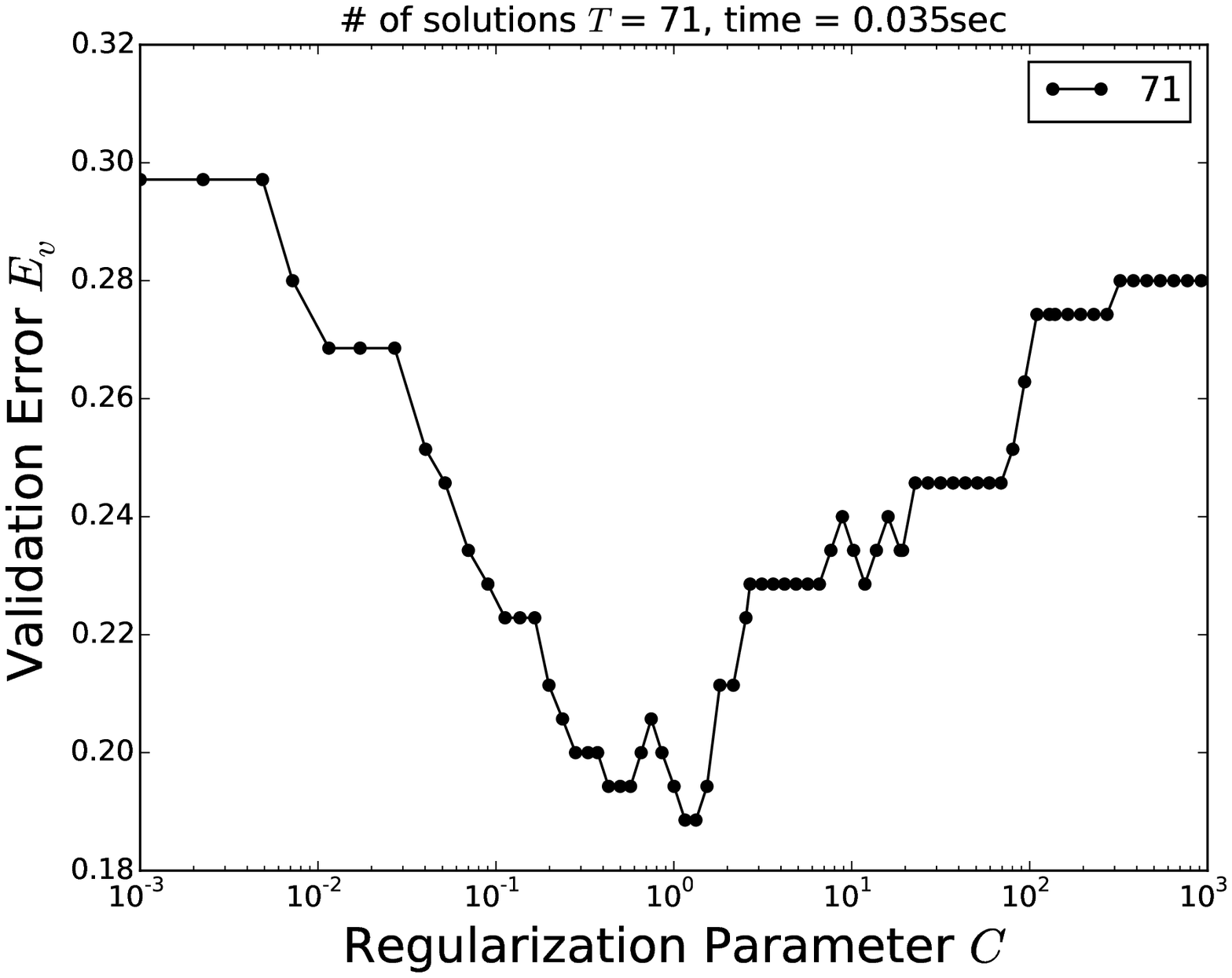} \\
    {\footnotesize (a) $\veps = 0.1$ without tricks}
    &
    {\footnotesize (b) $\veps = 0.05$ without tricks}
    &
    {\footnotesize (c) $\veps = 0.05$ with tricks 1 and 2}
   \end{tabular}
 \end{center}
 \caption{
  Illustrations of
  % Algorithm~\ref{alg:find.apprx}
  Algorithm 2
  on {\tt ionosphere} (D3) dataset
  for
 (a) {\bf op2} with $\veps=0.10$,
 (b) {\bf op2} with $\veps=0.05$
 and
 (c) {\bf op3} with $\veps=0.05$,
 respectively.
 \figurename~\ref{fig:illust}
 also shows the result for
{\bf op3} with $\veps=0.10$.
 %
 %An enlarged view of the shaded part in (a) is given in
 %\figurename~\ref{fig:ion_lbve}  in Appendix \ref{app:extensions}.
 }
 \label{fig:ion_sec6}
\end{figure*}

 \section{Experiments}
\label{sec:exp}
In this section
we present experiments
for illustrating the proposed methods.
%
%\paragraph{Datasets}
%
Table~\ref{tab:datasets}
summarizes the datasets
used in the experiments.
They are taken from
libsvm dataset repository
\cite{Chang2011a}.
All the input features except D9 and D10 were standardized to
$[-1, 1]$\footnote{
We use D9 and D10 as they are for exploiting sparsity.
}.
%
%In a setup with hold-out validation set,
%the whole instances is randomly dividid into
%a training
%and
%a validation sets
%in roughly equal size.
%
For illustrative results,
the instances were randomly divided into
a training
and
a validation sets
in roughly equal sizes.
For quantitative results,
we used 10-fold CV.
%
%In both cases,
%a dataset was divided
%such that
%the fraction of positive and negative instances in each subset is roughly equal.
%
%
%\paragraph{Loss function}
%As an example of the loss function,
%we used Huber hinge loss
%(e.g., \cite{Chapelle2007}).
% defined as
% \begin{align*}
%  \ell(y, w^\top x)
%  :=
%  \mycase{
%  \frac{1}{2} - y w^\top x
%  &
%  \text{for }
%  y w^\top x < 0,
%  \\
%  \frac{1}{2} (1 - y w^\top x)^2
%  &
%  \text{for }
%  y w^\top x \in [0, 1],
%  \\
%  0
%  &
%  \text{for }
%  y w^\top x > 1.
%  }
% \end{align*}
%
%
%\footnote{
%In our preliminary experiments (not shown),
%we have also used two other loss functions:
%$L_2$ hinge SVM
%$\ell(y, w^\top x) := \max\{0, (1 - y w^\top x)\}^2$
%and
%logistic loss
%$\ell(y, w^\top x) := \log(1 + \exp(- y w^\top x))$,
%and the results were qualitatively similar to the results
%on Huber hinge loss.}.
%%
%Note that
%the proposed method cannot be applied to
%most commonly used
%$L_1$ hinge loss
%$\ell(y, w^\top x) := \max\{0, 1 - y w^\top x\}$
%because it is not differentiable.
%%
%Huber hinge loss is useful
%because it behaves much like
%$L_1$ hinge loss,
%but is still differentiable.
%
%\paragraph{Loss function and optimization solver}
%
We used Huber hinge loss
(e.g., \cite{Chapelle2007})
which is convex
and
subdifferentiable with respect to the second argument.
The proposed methods are free from the choice of optimization solvers.
In the experiments,
we used an optimization solver
described in
\cite{Lin2008},
which is also implemented in
well-known
\emph{liblinear}
software
\cite{Fan2008}.
%Since
%we needed to modify some parts of the solver
%(e.g.,
%for the use of Huber hinge loss,
%for introducing additional stopping criteria on an approximate solution,
%etc.),
%we implemented the code by ourselves\footnote{
%We carefully checked the performances of our implementation
%by comparing it with
%\emph{liblinear}
%solver
%in several common setups.
%
%Other than miscellaneous additional properties particularly designed for our methods,
%A difference between
%our implementation
%and
%\emph{liblinear}
%is that
%we used
%\emph{eigen}\cite{eigenweb}
%for numerical linear algebra computation.
%
%Based on our comparisons,
%it makes our code slightly faster than original
%\emph{liblinear}.
%
%Our own code is provided as a supplementary material,
%and it will be put in public domain after the paper is accepted.
%}.
Our slightly modified code (for adaptation to Huber hinge loss)
is provided as a supplementary material,
and it will be put in public domain after the paper is accepted.
%
%In the experiments,
%we regarded the solutions obtained
%by using the same stopping criterion as
%\emph{liblinear}
%as optimal,
%and its gradient were assumed to be zero
%(although it is not always exactly zero due to numerical reasons).
%
%For approximate solutions,
%we precisely took into account the fact that the gradient is nonzero,
%and used it when computing the lower and the upper bound in
%\eq{eq:DS-approx}.
%
Whenever possible,
we used
\emph{warm-start} approach,
i.e.,
when we trained a new solution,
we used the closest solutions trained so far
(either approximate or optimal ones)
as the initial starting point of the optimizer.
All the computations were conducted by using a single core of an HP workstation
Z800 (Xeon(R) CPU X5675 (3.07GHz), 48GB MEM).
In all the experiments,
we set
$C_\ell = 10^{-3}$
and
$C_u = 10^{3}$.

\paragraph{Results on problem 1}
We applied Algorithm~\ref{alg:compute.eps} in \S\ref{sec:algorithm}
to
a set of solutions
obtained by
1) grid-search,
2) Bayesian optimization (with expected improvement acquisition function),
and
3) our own method that exploits information on the CV error lower bound available during the search process.
\figurename~\ref{fig:T_sec6} illustrates the results on three datasets,
where
we see how
the approximation level
$\veps$
in the vertical axis
changes
as the number of solutions ($T$ in our notation) increases.
The red plots indicate the results of grid-search.
As we increase the grid points,
the approximation level
$\veps$
was tended to be improved.
The blue plots indicate the results of Bayesian Optimization (BO).
Since BO tends to focus its search on a small region of the regularization parameter,
it was difficult to tightly bound the approximation level.
The green plots indicate the result of the third option,
where
we sequentially computed
a solution
whose validation error lower bound is smallest
based on the information obtained so far.
The results suggest that
this naive approach seems to offer slight improvement from grid-search.

%  \subsection{Problem 1: Computing approximation quality $\veps$ from the given set of solutions}
%
%$\cW = \{ {w}_{C_1}, \ldots, {w}_{C_T} \}$ options :
%\begin{itemize}
%  \item {\bf op1'} :
%  from $\cW$ obtained in process of grid search in the logarithmic scale.
%  \item {\bf op2'} :
%  from $\cW$ obtained in process of minimizing $E_{kCV}$ by bayesian optimization with gaussian process and
%   the acquisition function is Expected Improvement \cite{Snoek2012}.
%  \item {\bf op3'} :
%  from $\cW$ obtained in process of solving \eq{eq:the-class-of-problems} for
%  $C = $ the midpoint of the interval of the current $LB(E_v^{*})$ in turn.
%\end{itemize}
%
%\subsubsection{Illustrations}
%\label{subsec:N-results}
%For the purpose of illustration, we plot examples of changes of
%approximation quality $\veps$
%in the number of elements of a given set of solutions $T$ in CV.
%\figurename~\ref{fig:T_sec6}
%shows changes of $\veps$ on
%D2, D3 and D4 datasets
%for two options and $T$.
%%
%We see that
%{\bf op1'} can significantly reduce $\veps$ as $N$ increases.
%On the other hand, {\bf op2'} cannot reduce $\veps$
%since trained regularization parameters in the process of bayesian optimization
%do not spread over the interval of regularization parameter.
%%
%Also the tendency can be seen on other datasets.

% \subsection{Problem 2: Finding $\veps$-approximate regularization parameter}

\paragraph{Results on problem 2}
We applied Algorithm 2 to benchmark datasets
for demonstrating theoretically guaranteed choice of a regularization parameter
is possible with reasonable computational costs.
Besides the algorithm presented in \S\ref{sec:algorithm},
we also tested a variant described in supplementary Appendix~\ref{app:extensions}.
Specifically,
we have three algorithm options.
In the first option
({\bf op1}),
we used optimal solutions
$\{w^*_{\tilde{C}_t}\}_{t \in [T]}$
for computing CV error lower bounds.
In the second option
({\bf op2}),
we instead used approximate solutions
$\{\hat{w}_{\tilde{C}_t}\}_{t \in [T]}$.
In the last option
({\bf op3}),
we additionally used speed-up tricks
described in supplementary Appendix~\ref{app:extensions}.
We considered four different choices of
$\veps \in \{0.1, 0.05, 0.01, 0\}$.
Note that
$\veps = 0$
indicates
the task of finding the exactly optimal regularization parameter.
In some datasets,
the smallest validation errors are less than 0.1 or 0.05,
in which cases
we do not report the results
(indicated as ``$E_v < 0.05$'' etc.).
%
%Furthermore,
%we encountered a few rare cases
%where
%more than 100,000 solutions must be computed
%for finding the exactly optimal regularization parameter ($\veps = 0$).
%%
%In those cases,
%we quit the computation and the results are not presented
%(indicated as ``$T>{\rm 100K}$'').
%
In trick1,
we initially computed solutions at
four different regularization parameter values
evenly allocated in
$[10^{-3}, 10^3]$
in the logarithmic scale.
In trick2,
the next regularization parameter
$\tilde{C}_{t+1}$
was set by replacing
$\veps$ in
\eq{eq:nextC}
with
$1.5\veps$
(see supplementary Appendix~\ref{app:extensions}).

For the purpose of illustration,
we plot examples of validation error curves
in several setups.
\figurename~\ref{fig:ion_sec6}
shows the validation error curves of
{\tt ionosphere} (D3) dataset
for several options and $\veps$.

%\subsubsection{Computational costs}
%\label{subsec:exp-results}
%
Next,
we report the results on computational costs in CV setups.
Table~\ref{tab:results}
shows
the number of optimization problems we actually solved in the algorithm
(which is denoted as $T$),
and the total computation time in seconds.
The computational costs of the methods
mostly depend on
$T$.
As is evident from the algorithm description in
\S\ref{sec:algorithm},
$T$ gets smaller
as
$\veps$
increases.
Two tricks in
supplementary Appendix~\ref{app:extensions}
seem to be helpful
in most cases
for
reducing $T$.
In addition,
we see the advantage of using approximate solutions
by comparing the computation times of
{\bf op1}
and
{\bf op2},
although approximate solutions can be only used for
$\veps \neq 0$.
Overall,
the results suggest that
the proposed algorithm allows us to
find theoretically guaranteed approximate regularization parameters
with reasonable costs
except for $\veps = 0$ cases.
For example,
the algorithm found an
$\veps = 0.01$ approximate regularization parameter
within a minute
in 10-fold CV
for a dataset with more than 50000 instances
(see the results on D10 for $\veps = 0.01$ with {\bf op2} and {\bf op3} in Table~\ref{tab:results}).

\begin{table}[t]
   \vspace*{-3mm}
 \begin{center}
   \caption{Computational costs.
   For each of the three options
   and
   $\veps \in \{0.10, 0.05, 0.01, 0\}$,
   the number of optimization problems solved (denoted as $T$)
   and the total computational costs (denoted as time)
   are listed.
   Note that,
   for {\bf op2},
   there are no results for $\veps = 0$.
   }
   \label{tab:results}
   \begin{scriptsize}
    \begin{tabular}{r||l|r|r||r|r||r|r||l|r|r||r|r||r|r}
     & & \multicolumn{2}{|c||}{\bf op1} & \multicolumn{2}{|c||}{\bf op2} &  \multicolumn{2}{|c||}{\bf op3}  & & \multicolumn{2}{|c||}{\bf op1} & \multicolumn{2}{|c||}{\bf op2} &  \multicolumn{2}{|c}{\bf op3}  \\
     & & \multicolumn{2}{|c||}{(using $w^*_{\tilde{C}}$)} & \multicolumn{2}{|c||}{(using $\hat{w}_{\tilde{C}}$)} &  \multicolumn{2}{|c||}{(using tricks)} & & \multicolumn{2}{|c||}{(using $w^*_{\tilde{C}}$)} & \multicolumn{2}{|c||}{(using $\hat{w}_{\tilde{C}}$)} &  \multicolumn{2}{|c}{(using tricks)}  \\ \cline{3-15}
     \multirow{2}{*}{$\veps~~$}&  & \multirow{2}{*}{$T$} & time & \multirow{2}{*}{$T$} & time  & \multirow{2}{*}{$T$} & time & & \multirow{2}{*}{$T$} & ~time~ & \multirow{2}{*}{$T$} & ~time~ & \multirow{2}{*}{$T$} & ~time~  \\
     & & & (sec) & & (sec) & & (sec) & & & ~(sec)~ & & ~(sec)~ & & ~(sec)~ \\ \hline \hline
     0.10 & \multirow{4}{2pt}{$\!\! {\rm D1}$}&   30 & 0.068 &   32 & 0.031 &  33 & 0.041 & \multirow{4}{3pt}{$\!\! {\rm D6}$}   & 92 & 1.916 &   93 & 0.975 &  62 & 0.628   \\ \cline{1-1} \cline{3-8} \cline{10-15}
     0.05 & & 68 & 0.124 &   70 & 0.061 &  57 & 0.057 & &    207 & 4.099 &  209 & 2.065 & 123 & 1.136 \\ \cline{1-1} \cline{3-8} \cline{10-15}
     0.01 & & 234 & 0.428 &  324 & 0.194 & 205 & 0.157&  & 1042 & 16.31 & 1069 & 9.686 & 728 & 5.362 \\ \cline{1-1} \cline{3-8} \cline{10-15}
        0 & & 442 & 0.697 & \multicolumn{2}{|c||}{N.A.} & 383 & 0.629 & &   4276 & 57.57 & \multicolumn{2}{|c||}{N.A.}&2840 & 44.68 \\ \hline \hline

     0.10 & \multirow{4}{2pt}{$\!\! {\rm D2}$}&  221 & 0.177 &  223 & 0.124 & 131  & 0.084 & \multirow{4}{3pt}{$\!\! {\rm D7}$} &    289 & 8.492 &  293 & 5.278 & 167 & 3.319 \\ \cline{1-1} \cline{3-8} \cline{10-15}
     0.05 & & 534 & 0.385 &  540 & 0.290 & 367  & 0.218 &  &    601 & 16.18 &  605 & 9.806 & 379 & 6.604 \\ \cline{1-1} \cline{3-8} \cline{10-15}
     0.01 & & 1503 & 0.916 & 2183 & 0.825 & 1239 & 0.623 & &   2532 & 57.79 & 2788 & 35.21 &1735 & 24.04 \\ \cline{1-1} \cline{3-8} \cline{10-15}
        0 & & 10939 & 6.387 & \multicolumn{2}{|c||}{N.A.} & 6275 & 3.805 & & 67490  & 1135  & \multicolumn{2}{|c||}{N.A.}&42135& 760.8 \\ \hline \hline

    0.10 & \multirow{4}{2pt}{$\!\! {\rm D3}$} &    61 & 0.617 &   62 & 0.266 &  43 & 0.277 & \multirow{4}{3pt}{$\!\! {\rm D8}$} &     72 & 0.761 &   74 & 0.604 &  66 & 0.606 \\ \cline{1-1}  \cline{3-8} \cline{10-15}
    0.05 &  & 123  & 1.073 &  129 & 0.468 &  73 & 0.359 & &    192 & 1.687 &  195 & 1.162 & 110 & 0.926  \\ \cline{1-1}  \cline{3-8} \cline{10-15}
    0.01 &  & 600  & 4.776 &  778 & 0.716 & 270 & 0.940 &  &   1063 & 8.257 & 1065 & 6.238 & 614 & 4.043  \\ \cline{1-1}  \cline{3-8} \cline{10-15}
       0 &  &5412  & 26.39 & \multicolumn{2}{|c||}{N.A.}& 815 & 6.344 &  &  34920 & 218.4 & \multicolumn{2}{|c||}{N.A.} &15218& 99.57 \\ \hline \hline

    0.10 & \multirow{4}{2pt}{$\!\! {\rm D4}$} &     27 & 0.169 &   27 & 0.088 &  23 & 0.093 & \multirow{4}{3pt}{$\!\! {\rm D9}$} &    134 & 360.2 &  136 & 201.0 &  89 & 74.37 \\ \cline{1-1}  \cline{3-8} \cline{10-15}
    0.05 &  &    64 & 0.342 &   65 & 0.173 &  47 & 0.153 & &    317 & 569.9 &  323 & 280.7 & 200 & 128.5 \\ \cline{1-1}  \cline{3-8} \cline{10-15}
    0.01 &  &   167 & 0.786 &  181 & 0.418 & 156 & 0.399 & &   1791 & 2901  & 1822 & 1345  &1164 & 657.4 \\ \cline{1-1}  \cline{3-8} \cline{10-15}
       0 &  &   342 & 1.317 & \multicolumn{2}{|c||}{N.A.} & 345 & 1.205 & &  85427 &106937 & \multicolumn{2}{|c||}{N.A.} & 63300& 98631 \\ \hline \hline

    0.10 & \multirow{4}{2pt}{$\!\! {\rm D5}$} &     62 & 0.236 &   63 & 0.108 &  45 & 0.091 & \multirow{4}{3pt}{$\!\!\!\! {\rm D10}$} & \multicolumn{2}{|c||}{$E_v<0.10$} & \multicolumn{2}{|c||}{$E_v<0.10$} & \multicolumn{2}{|c}{$E_v<0.10$} \\ \cline{1-1}  \cline{3-8} \cline{10-15}
    0.05 &  &   108 & 0.417 &  109 & 0.171 &  77 & 0.137 & & \multicolumn{2}{|c||}{$E_v<0.05$} & \multicolumn{2}{|c||}{$E_v<0.05$} & \multicolumn{2}{|c}{$E_v<0.05$}  \\ \cline{1-1}  \cline{3-8} \cline{10-15}
    0.01 &  &   421 & 1.201 &  440 & 0.631 & 258 & 0.401 & &    157 & 81.75 &  162 & 31.02 &  114  & 36.81 \\ \cline{1-1}  \cline{3-8} \cline{10-15}
       0 &  &  2330 & 4.540 & \multicolumn{2}{|c||}{N.A.} & 968 & 2.451 &  & 258552 & 85610 & \multicolumn{2}{|c||}{N.A.} & 42040 & 23316 %\\ \hline \hline

    \end{tabular}
   \end{scriptsize}
 \end{center}
 \vspace*{-5mm}
\end{table}

\begin{table}[h]
 \vspace*{-5mm}
 \begin{center}
  \caption{Benchmark datasets used in the experiments.}
  \label{tab:datasets}
  \vspace*{1mm}
  \begin{scriptsize}
  \begin{tabular}{l||c|r|r||l||c|r|r}
 &dataset name&~~sample size~~&input dimension & &dataset name&~~sample size~~&input dimension\\\hline\hline
   D1&{\tt heart}&270~~~~~~&13~~~~~~~~~& D6&{\tt german.numer}&1000~~~~~~&24~~~~~~~~~\\ \hline
   D2&{\tt liver-disorders}&345~~~~~~&6~~~~~~~~~&D7&{\tt svmguide3}&1284~~~~~~&21~~~~~~~~~\\ \hline
   D3&{\tt ionosphere}&351~~~~~~&34~~~~~~~~~&D8&{\tt svmguide1}&7089~~~~~~&4~~~~~~~~~\\ \hline
   D4&{\tt australian}&690~~~~~~&14~~~~~~~~~&D9&{\tt a1a}&32561~~~~~~&123~~~~~~~~~\\ \hline
   D5&{\tt diabetes}&768~~~~~~&8~~~~~~~~~&D10&{\tt w8a}&64700~~~~~~&300~~~~~~~~~\\ %\hline
   % D6&{\tt german.numer}&1000~~~~~~&24~~~~~~~~~\\\hline
   % D7&{\tt svmguide3}&1284~~~~~~&21~~~~~~~~~\\\hline
   % D8&{\tt svmguide1}&7089~~~~~~&4~~~~~~~~~\\\hline
   % D9&{\tt a1a}&32561~~~~~~&123~~~~~~~~~\\\hline
   % D10&{\tt w8a}&64700~~~~~~&300~~~~~~~~
  \end{tabular}
  \end{scriptsize}
 \end{center}
 \vspace*{-5mm}
\end{table}

\section{Conclusions and future works}
\label{sec:conc}
We presented a novel algorithmic framework
for computing CV error lower bounds
as a function of the regularization parameter.
The proposed framework can be used for
a theoretically guaranteed choice of a regularization parameter.
Additional advantage of this framework is that
we only need to compute
a set of sufficiently good approximate solutions
for obtaining such a theoretical guarantee,
which is computationally advantageous.
As demonstrated in the experiments,
our algorithm is practical
in the sense that
the computational cost is reasonable
as long as the approximation quality
$\veps$
is not too close to 0.
An important future work is to extend the approach
to multiple hyper-parameters tuning setups.

% \newpage
% \begin{small}
% \bibliographystyle{unsrt}
% \bibliography{./library}
% \end{small}

\appendix
\section{Proof of Lemma~\ref{lemm:decision-score-bounds-with-approximate-solutions}}
\label{app:proofs}
In this section we prove Lemma~\ref{lemm:decision-score-bounds-with-approximate-solutions}.
First
we present
two propositions
which are used of proving
Lemma
\ref{lemm:decision-score-bounds-with-approximate-solutions}.

\begin{prop}
 \label{prop:VI}
 Consider the following general problem:
 \begin{align}
  \label{eq:general.convex.constrained.problem}
  \min_z ~ \phi(z) ~~~ {\rm s.t.} ~ z \in \cZ,
 \end{align}
 where
 $\phi: \cZ \to \RR$
 is a subdifferentiable convex function and $\cZ \subset \RR^d$ is a convex set.
 Then a solution $z^*$ is the optimal solution of
 \eq{eq:general.convex.constrained.problem}
 if and only if there exists a subgradient $\xi \in \partial \phi(z^*)$
 such that
 \begin{align*}
  \xi ^\top (z^* - z) \le 0, ~~~ \forall ~ z \in \cZ,
 \end{align*}
 where $\partial \phi(z^*)$ is the set of all subgradients of convex function $\phi$ at $z = z^*$.
\end{prop}
See,
for example,
Proposition B.24 in \cite{Bertsekas1999}
% Proposition 2.1.2 in \cite{Bertsekas1999}
% or
% Section 4.2.4 in \cite{Boyd2004}
for the proof of
Proposition~\ref{prop:VI}.

\begin{prop}
 \label{prop:ball}
 Let
 $p, q \in \RR^d$
 be arbitrary $d$-dimensional vectors
 and
 $r > 0$
 be an arbitrary positive constant.
 Then,
 the solutions of the following optimization problem
 can be explicitly obtained as follows:
 \begin{align}
  \label{eq:ball-L}
  p^\top q - \|p\| r
  &=
  \min_{z \in \RR^d}~p^\top z
  ~~~
  {\rm s.t.}
  ~
  \|z - q\|^2 \le r^2 ,
  \\
  \label{eq:ball-U}
  p^\top q + \|p\| r
  &=
  \max_{z \in \RR^d}~p^\top z
  ~~~
  {\rm s.t.}
  ~
  \|z - q\|^2 \le r^2 .
 \end{align}
\end{prop}

\begin{proof}[Proof of Proposition \ref{prop:ball}]
 Using a Lagrange multiplier
 $\lambda > 0$,
 the problem
 \eq{eq:ball-L}
 is rewritten as
 \begin{align*}
  &
  \min_{z \in \RR^d}~p^\top z
  ~~~
  {\rm s.t.}
  ~
  \|z - q\|^2 \le r^2
  \\
  =
  &
  \min_{z \in \RR^d}
  \max_{\lambda > 0}
  \big(
  p^\top z + \lambda (\|z - q\|^2 \le r^2)
  \big)
  \\
  =
  &
  \max_{\lambda > 0}
  \big(
  - \lambda r^2
  +
  \min_{z}
  \big(
  \lambda \|z - p\|^2 + p^\top z
  \big)
  \big)
  \\
  =
  &
  \max_{\lambda > 0}
  ~
  H(\lambda) :=
  \big(
  - \lambda r^2
  - \frac{\|p\|^2}{4 \lambda}
  +
  p^\top q
  \big),
 \end{align*}
 where
 $\lambda$
 is strictly positive
 because the constraint
 $\|p - q\|^2 \le r^2$
 is strictly active
 at the optimal solution.
 By letting
 $\partial H(\lambda)/\partial \lambda = 0$,
 the optimal
 $\lambda$
 is written as
 \begin{align*}
  \lambda^* := \frac{\|p\|}{2 r}
  =
  \arg \max_{\lambda > 0} ~ H(\lambda).
 \end{align*}
 Substituting
 $\lambda^*$
 into
 $H(\lambda)$,
 \begin{align*}
  p^\top q - \|p\| r = \max_{\lambda > 0}~H(\lambda).
 \end{align*}
 The upper bound of
 $p^\top z$
 in
 \eq{eq:ball-U}
 can be shown similarly.
\end{proof}

\begin{proof}[Proof of Lemma~\ref{lemm:decision-score-bounds-with-approximate-solutions}]
 From
 Proposition~\ref{prop:VI},
 the optimal solution
 $w^*_C$
 satisfies
 \begin{align}
%  \nonumber
%  &
%  \nabla J_C(w^*_C)^\top (w^*_C - \hat{w}_{\tilde{C}}) \le 0
%  \\
%  \label{eq:prf-b}
%  ~\Leftrightarrow~
%  &
\label{eq:prf-b}
  \left( w^*_C + C \sum_{i \in [n]} \xi_i(w^*_C) \right)^\top (w^*_C - \hat{w}_{\tilde{C}}) \le 0,
 \end{align}
 where
% $\xi_i(w^*_C) \in \partial \ell_i(w^*_C) $ and $\partial  \ell_i(w^*_C)$ is the set of all subgradients of convex loss function $\ell_i$ at $w = w^*_C$ for any $i \in [n]$ .
$\xi_i(w^*_C) $ is a subgradient of $\ell_i$ at $w = w^*_C$ for any $i \in [n]$ .

 Since from
 $\ell_i$
 is convex
 for any
 $i \in [n]$ and the definition of a subgradient,
 we have the following two inequalities:
 \begin{align*}
  &
  \ell_i(w^*_C)
  \ge
  \ell_i(\hat{w}_{\tilde{C}})
  +
  \xi_i(\hat{w}_{\tilde{C}})
  ^\top(w^{*}_C - \hat{w}_{\tilde{C}}).
  \\
  &
  \ell_i(\hat{w}_{\tilde{C}})
  \ge
  \ell_i(w^*_C)
  +
  \xi_i(w^*_C)^\top
  (\hat{w}_{\tilde{C}} - w^{*}_C).
 \end{align*}
 Combining these two inequalities,
 we have
 \begin{align}
  \label{eq:prf-a}
  \xi_i(w^*_C) ^\top(w^{*}_C - \hat{w}_{\tilde{C}})
  \ge
  \xi_i(\hat{w}_{\tilde{C}}) ^\top(w^{*}_C - \hat{w}_{\tilde{C}}).
 \end{align}
 Substituting
 \eq{eq:prf-a}
 into
 \eq{eq:prf-b},
 \begin{align}
  \label{eq:prf-d}
  w^{*\top}_C (w^*_C - \hat{w}_{\tilde{C}})
  +
  C \sum_{i \in [n]}
  \xi_i(\hat{w}_{\tilde{C}}) ^\top(w^{*}_C - \hat{w}_{\tilde{C}})
  \le 0.
 \end{align}
 From
 \eq{eq:gradient},
 \begin{align}
  \label{eq:prf-c}
  \sum_{i \in [n]}
  \xi_i(\hat{w}_{\tilde{C}})
  =
  \frac{1}{\tilde{C}}
  \Big(
  g(\hat{w}_{\tilde{C}})
  -
  \hat{w}_{\tilde{C}}
  \Big) .
 \end{align}
 Substituting
 \eq{eq:prf-c}
 into
 \eq{eq:prf-d},
 \begin{align*}
  &
  w^{*\top}_C (w^*_C - \hat{w}_{\tilde{C}})
  +
  \frac{C}{\tilde{C}}
  \Big(
  g(\hat{w}_{\tilde{C}})
  -
  \hat{w}_{\tilde{C}}
  \Big)
  ^\top(w^{*}_C - \hat{w}_{\tilde{C}})
  \le 0
  \\
  \Leftrightarrow~
  &
  \Big\|
  w^*_C -
  \frac{1}{2}
  \Big(
  \hat{w}
  -
  \frac{C}{\tilde{C}}
  (g(\hat{w}) - \hat{w})
  \Big)
  \Big\|^2
  \le
  \Big(
  \frac{1}{2}
  \Big\|
  \hat{w} + \frac{C}{\tilde{C}}(g(\hat{w}) - \hat{w})
  \Big\|
  \Big)^2.
 \end{align*}
 The lower bound
 ${LB}(w^{*\top}_C x^\prime_i | \hat{w}_{\tilde{C}})$
 is given by solving the following optimization problem:
 \begin{align}
  \label{eq:prf-f}
  &
  \min_{w^*_C}
  ~~~
  w^{*\top}_C x^\prime_i
  ~~~ ~~~
  {\rm s.t.}
  ~~~
  \Big\|
  w^*_C -
  \frac{1}{2}
  \Big(
  \hat{w}
  -
  \frac{C}{\tilde{C}}
  (g(\hat{w}) - \hat{w})
  \Big)
  \Big\|^2
  \le
  \Big(
  \frac{1}{2}
  \Big\|
  \hat{w} + \frac{C}{\tilde{C}}(g(\hat{w}) - \hat{w})
  \Big\|
  \Big)^2.
 \end{align}
 Using
 Proposition
 \ref{prop:ball},
 the solution of
 \eq{eq:prf-f}
 is given as
 \begin{align*}
  {LB}(w^{*\top}_C x^\prime_i | \hat{w}_{\tilde{C}})
  &=
  \frac{1}{2}
  x^{\prime \top}_i
  \Big(
  \hat{w}
  -
  \frac{C}{\tilde{C}}
  (g(\hat{w}) - \hat{w})
  \Big)
  -
  \|
  x^\prime_i
  \|
  \Big\|
  \frac{1}{2}
  \Big(
  \hat{w} + \frac{C}{\tilde{C}}(g(\hat{w}) - \hat{w})
  \Big)
  \Big\|
  \\
  &
  \le
  \frac{1}{2}
  x^{\prime \top}_i
  \Big(
  \hat{w}
  -
  \frac{C}{\tilde{C}}
  (g(\hat{w}) - \hat{w})
  \Big)
  -
  \frac{1}{2}
  \|
  x^\prime_i
  \|
  \Big(
  \Big|
  1 - \frac{C}{\tilde{C}}
  \Big|
  \|
  \hat{w}
  \|
  +
  \frac{C}{\tilde{C}}
  \|
  g(\hat{w})
  \|
  \Big)
  \\
  &
  =
  \mycase{
  \phantom{+} \alpha(\hat{w}_{\tilde{C}}, x^\prime_i)
  -
  \frac{1}{\tilde{C}}
  (
  \beta(\hat{w}_{\tilde{C}}, x^\prime_i)
  +
  \gamma(g(\hat{w}_{\tilde{C}}), x^\prime_i)
  ) C,
  % \\
  \hspace*{15pt}
  \text{if }
  C \ge \tilde{C},
  \\
  - \beta(\hat{w}_{\tilde{C}}, x^\prime_i)
  +
  \frac{1}{\tilde{C}}
  (
  \alpha(\hat{w}_{\tilde{C}}, x^\prime_i) + \delta(g(\hat{w}_{\tilde{C}}), x^\prime_i)
  ) C,
  % \\
  \hspace*{15pt}
  \text{if }
  C < \tilde{C}.
  }
 \end{align*}
 Similarly,
 the upper bound
 ${UB}(w^{*\top}_C x^\prime_i | \hat{w}_{\tilde{C}})$
 is given by solving the following optimization problem
 \begin{align}
  \label{eq:prf-g}
  \max_{w^*_C}
  ~~~
  w^{*\top}_C x^\prime_i
  ~~~ ~~~
  {\rm s.t.}
  ~~~
  \Big\|
  w^*_C -
  \frac{1}{2}
  \Big(
  \hat{w}
  -
  \frac{C}{\tilde{C}}
  (g(\hat{w}) - \hat{w})
  \Big)
  \Big\|^2
  \le
  \Big(
  \frac{1}{2}
  \Big\|
  \hat{w} + \frac{C}{\tilde{C}}(g(\hat{w}) - \hat{w})
  \Big\|
  \Big)^2,
 \end{align}
 and
 the solution of
 \eq{eq:prf-g}
 is given as
 \begin{align*}
  {UB}(w^{*\top}_C x^\prime_i | \hat{w}_{\tilde{C}})
  &=
  \frac{1}{2}
  x^{\prime \top}_i
  \Big(
  \hat{w}
  -
  \frac{C}{\tilde{C}}
  (g(\hat{w}) - \hat{w})
  \Big)
  +
  \|
  x^\prime_i
  \|
  \Big\|
  \frac{1}{2}
  \Big(
  \hat{w} + \frac{C}{\tilde{C}}(g(\hat{w}) - \hat{w})
  \Big)
  \Big\|
  \\
  &
  \ge
  \frac{1}{2}
  x^{\prime \top}_i
  \Big(
  \hat{w}
  -
  \frac{C}{\tilde{C}}
  (g(\hat{w}) - \hat{w})
  \Big)
  +
  \frac{1}{2}
  \|
  x^\prime_i
  \|
  \Big(
  \Big|
  1 - \frac{C}{\tilde{C}}
  \Big|
  \|
  \hat{w}
  \|
  +
  \frac{C}{\tilde{C}}
  \|
  g(\hat{w})
  \|
  \Big)
  \\
  &
  =
  \mycase{
  -\beta(\hat{w}_{\tilde{C}}, x^\prime_i)
  +
  \frac{1}{\tilde{C}}
  (
  \alpha(\hat{w}_{\tilde{C}}, x^\prime_i)
  +
  \delta(g(\hat{w}_{\tilde{C}}), x^\prime_i)
  ) C,
  % \\
  \hspace*{15pt}
  \text{if }
  C \ge \tilde{C},
  \\
  \phantom{+}\alpha(\hat{w}_{\tilde{C}}, x^\prime_i)
  -
  \frac{1}{\tilde{C}}
  (
  \beta(\hat{w}_{\tilde{C}}, x^\prime_i) + \gamma(g(\hat{w}_{\tilde{C}}), x^\prime_i)
  ) C,
  % \\
  \hspace*{15pt}
  \text{if }
  C < \tilde{C}.
  }
 \end{align*}
\end{proof}

% ---------------------------------------------------------------------------
%\setcounter{theo}{9}
% ---------------------------------------------------------------------------
\begin{rema}
We note that
the idea of using
Propositions~\ref{prop:VI} and \ref{prop:ball}
for proving
Lemma~\ref{lemm:decision-score-bounds-with-approximate-solutions}
is inspired from recent studies on
safe screening~\cite{ElGhaoui2012,Xiang2011,Ogawa2013,Liu2014,Wang2014}.
Safe screening has been introduced in the context of sparse modeling.
It allows us to identify sparse features or instances
before actually solving the optimization problem.
A key technique used in those studies is to bound
Lagrange multipliers at the optimal solution
(Lagrange multiplier values
at the optimal solution tell us which features or instances are active or non-active)
in somewhat similar way as we did in \S\ref{sec:theory}.
Our main contribution is to borrow this idea for
representing
a validation error lower bound
as a function of the regularization parameter,
and show that it can be used for finding an approximately optimal regularization parameter with theoretical guarantee.
\end{rema}

\begin{figure}
 % \begin{center}
   \floatbox[{\capbeside\thisfloatsetup{capbesideposition={right,top},capbesidewidth=65mm}}]{figure}[\FBwidth]
  { \hspace*{-7mm}
  \caption{An illustrative example of Algorithm 2 behavior.
  The blue real lines represent the validation error lower bound.
  The red chained lines and green dashed lines indicate
  the current best validation error upper bound
  $E^{\rm best}_v$
  and
  $E^{\rm best}_v - \veps$,
  respectively.
  If the blue validation error lower bound falls below the green ones,
  %it indicates that
  the validation error can be smaller
  by
  $\eps$
  than the current best.
  In such a case,
  the algorithm computes the next approximate solution,
  and update the validation error lower bound
  based on the new approximate solution.
  The plot is an enlarged view of
  the region
  from $\tilde{C}_{13}$ to $\tilde{C}_{17}$
  in
  \figurename~\ref{fig:ion_sec6} (a)
  in \S\ref{sec:exp}. }
  }
  {
    \hspace{5mm}
    \includegraphics[width=0.45\textwidth]{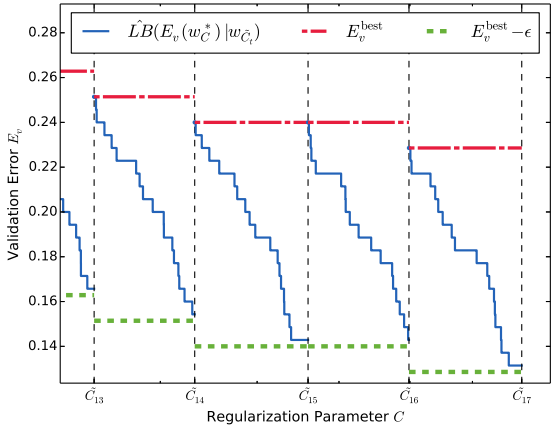}
  }
  \label{fig:ion_lbve}
 % \end{center}
\end{figure}

\section{Details of the speed-up tricks for finding an $\veps$-approximate regularization parameter}
\label{app:extensions}

In this appendix,
we first describe
two modifications of the basic algorithm for finding an $\veps$-approximate regularization parameter
presented in \S\ref{subs:find.apprx}
for further speed-up.
%
% Next,
% we extend the algorithm
% for computing
% an \emph{$\veps$-approximate regularization path}.
%in which the approximation quality is measured in terms of the validation errors.
%It is different from recent approaches in
%\cite{}
%in the sense that
%the approximation quality is measured in terms the objective function values
%in these existing studies,
%while
%our method
%allows us
%to guarantee the the approximation quality
%in terms of the validation errors.
%
% Finally,
% we briefly discuss how we can adapt the proposed methods to cross-validation setup.
%
%Due to the space limitation,
%we briefly describe these extensions here,
%and some details are presented in appendices.

% \subsection{Speed-up tricks}
% \label{subsec:trick}
%
\paragraph{Trick1}
The efficiency of the algorithm
depends on
how far one can step forward in each iteration.
We see
in
\eq{eq:nextC}
that
the step size
$\tilde{C}_{t+1} - \tilde{C}_t$
is large
if the current minimum validation error upper bound
$E_v^{\rm best}$
is small.
In other words,
the step size will be small until
we have sufficiently small
$E_v^{\rm best}$.
It suggests that,
if we can find
small enough
$E_v^{\rm best}$
at an \emph{earlier} stage of the algorithm,
we can reduce the total computational cost of the algorithm.
In order to find sufficiently small
$E_v^{\rm best}$
as early as possible,
we propose a simple heuristic approach,
where
we first roughly search over the entire range by a rough grid search.
%
%In experiments,
%we investigate the efficacy of this simple modification
%under the name of
%{\bf trk1}.
%

%  \subsection{Speed-up trick 2 (trk2)}
%\label{subsec:trick2}
\paragraph{Trick2}
Our next modification for speed-up is to use
\begin{align*}
 % &
 {LB}(E_v(w^*_C) | \hat{w}_{\tilde{C}_t}, \hat{w}_{\tilde{C}_{t+1}})
 % \\
 % &
 :=
 \max\{
 {LB}(E_v(w^*_C) | \hat{w}_{\tilde{C}_t}),
 {LB}(E_v(w^*_C) | \hat{w}_{\tilde{C}_{t+1}})
 \},
\end{align*}
for computing the validation error lower bound
in
$C \in [\tilde{C}_{t}, \tilde{C}_{t+1}]$.
It provides a tighter validation error lower bounds
than
using
${LB}(E_v(w^*_C) | \hat{w}_{\tilde{C}_t})$
alone,
meaning that larger step might be allowed
in each iteration.
However,
we cannot actually compute
${LB}(E_v(w^*_C) | \hat{w}_{\tilde{C}_{t+1}})$
before we fix
$\tilde{C}_{t+1}$.
We thus propose a simple trial-and-error approach.
Specifically,
we step forward a little bit further than \eq{eq:nextC}
when we select the next
$\tilde{C}_{t+1}$.
After we fix $\tilde{C}_{t+1}$,
we compute an approximate solution
$\hat{w}_{\tilde{C}_{t+1}}$
and then check
whether the validation error
$E_v(w^*_C)$
is not smaller
by $\veps$
than the current minimum
for
$C \in [\tilde{C}_t, \tilde{C}_{t+1}]$
by using now available
${LB}(E_v(w^*_C) | \hat{w}_{\tilde{C}_{t}}, \hat{w}_{\tilde{C}_{t+1}})$.
%
%In experiments,
%we investigate the efficacy of this approach
%under the name of
%{\bf trk2}.

% The pseudo-codes of these two tricks are presented in Appendix~\ref{app:extensions}.

% In this appendix,
% we describe the details of the speed-up tricks presented in
% \ref{subsec:trick}.
% %

% Algorithm~\ref{alg:op3}
Algorithm 3
is the pseudo-code of the proposed algorithm along with tricks 1 and 2.

\makeatletter
\renewcommand{\ALG@name}{Algorithm 3}
\renewcommand{\thealgorithm}{}
\makeatother
% --------------------------------------------------
% algorithm
% --------------------------------------------------
\begin{algorithm}[!ht]
  \caption{Finding an $\veps$-approximate regularization parameter
 with approximate solutions
 using tricks 1 and 2}
  \label{alg:op3}
  \begin{algorithmic}[1]
    \REQUIRE
      $\{(x_i, y_i)\}_{i \in [n]}$,
      $\{(x^\prime_i, y^\prime_i)\}_{i \in [n^\prime]}$,
      $C_l$,
      $C_u$,
      $\veps$,
      $m$,
      $\rho$
    \STATE $C^{\rm best} \lA C_l$, $E_v^{\rm best} \lA 1$
    \STATE $s \lA \frac{\log_{10}(C_u) - \log_{10}(C_l)}{m} $
    \FOR{ $h = 0$ to $m-1$ }
      \STATE $\bar{C}_{h} \lA 10^{(\log_{10}(C_l) + h \times s)} $
      \STATE $\hat{w}_{\bar{C}_h}$ $\lA$ solve \eq{eq:the-class-of-problems} approximately for $C = \bar{C}_h$ %and obtain $\hat{w}_{\bar{C}_h}$
      \IF{${UB}(E_v(w^*_{\bar{C}_h}) | \hat{w}_{\bar{C}_h}) < E_v^{\rm best} $}
        \STATE $E_v^{\rm best} \lA {UB}(E_v(w^*_{\bar{C}_h}) | \hat{w}_{\bar{C}_h})$, $C^{\rm best} \lA \bar{C}_h$
      \ENDIF
    \ENDFOR
    \STATE $\bar{C}_m \lA C_u$ , $t \lA 1$
    \FOR{ $h = 0$ to $m-1$ }
      \STATE $\tilde{C}_t \lA \bar{C}_{h}$ , $\hat{w}_{\tilde{C}_t} \lA \hat{w}_{\bar{C}_{h}}$
      \WHILE{$\tilde{C}_t \le \bar{C}_{h+1}$}
        \STATE Set ${C}^{\rm tmp}$ by \eq{eq:nextC.aggr} using $\hat{w}_{\tilde{C}_t}$
        \IF{${C}^{\rm tmp} > \bar{C}_{h+1}$ }
          \STATE Set $C^{\rm tmp}$ by \eq{eq:nextC.right} using $\hat{w}_{\tilde{C}_t}$
          \IF{${C}^{\rm tmp} > \bar{C}_{h+1}$ }
            \STATE break while loop
          \ENDIF
        \ENDIF
        \STATE $\hat{w}_{{C}^{\rm tmp}}$ $\lA$ solve \eq{eq:the-class-of-problems} approximately for $C = {C}^{\rm tmp}$
        \STATE Compute ${UB}(E_v(w^*_{{C}^{\rm tmp}}) | \hat{w}_{{C}^{\rm tmp}})$ by \eq{eq:valid-err-UB-tilde}.
        \IF{${UB}(E_v(w^*_{{C}^{\rm tmp}}) | \hat{w}_{{C}^{\rm tmp}}) < E_v^{\rm best} $}
          \STATE $E_v^{\rm best} \lA {UB}(E_v(w^*_{{C}^{\rm tmp}}) | \hat{w}_{{C}^{\rm tmp}})$
          \STATE $C^{\rm best} \lA {C}^{\rm tmp}$
        \ENDIF
        \STATE $r \lA 0$
        \STATE RecursiveCheck($\tilde{C}_t, {C}^{\rm tmp}, \hat{w}_{\tilde{C}_t}, \hat{w}_{{C}^{\rm tmp}}, r $)
        \STATE $\tilde{C}_{t+r+1} \lA {C}^{\rm tmp}$, $\hat{w}_{\tilde{C}_{t+r+1}} \lA \hat{w}_{{C}^{\rm tmp}}$
        \STATE $t \lA t+r+1$
      \ENDWHILE
    \ENDFOR
    \ENSURE
    $C^{\rm best} \in \cC(\veps)$.
    %$C^{\rm best}$.
    %an $\veps$-approximate regularization parameter $C^{\rm best}$;
  \end{algorithmic}
\end{algorithm}

There are two additional input parameters
$m \in \bN$
and
$\rho > 1$.
The former is used for trick1,
where we initially compute $m$ approximate solutions
for regularization parameter values evenly allocated in the interval
$[C_l, C_u]$
in the logarithmic scale.
Trick1 is described at
lines 2-9 in
% Algorithm~\ref{alg:op3}.
Algorithm 3.

\makeatletter
\renewcommand{\ALG@name}{Algorithm 4}
\renewcommand{\thealgorithm}{}
\makeatother
\begin{algorithm}[t]
 \caption{RecursiveCheck $( C(L), C(R), \hat{w}_{C(L)}, \hat{w}_{C(R)}, r)$ }
 \label{alg:rec.check}
 \begin{algorithmic}[]
  \STATE Compute $C^R(\hat{w}_{C(L)})$ in \eq{eq:nextC.left}.
  \STATE Compute $C^L(\hat{w}_{C(R)})$ in \eq{eq:nextC.right}.
      \IF{$C^L(\hat{w}_{C(R)}) < C^R(\hat{w}_{C(L)})$}
      \STATE return
    \ELSE
      \STATE $r \lA r+1$
      \STATE $\tilde{C}_{t+r} $ $\lA$ $ \frac{1}{2}(C^L(\hat{w}_{C(R)}) + C^R(\hat{w}_{C(L)}) ) $
      \STATE $\hat{w}_{\tilde{C}_{t+r}}$ $\lA$ solve \eq{eq:the-class-of-problems} approximately for $C = \tilde{C}_{t+r}$
      \IF{${UB}(E_v(w^*_{\tilde{C}_{t+r}}) | \hat{w}_{\tilde{C}_{t+r}}) < E_v^{\rm best} $}
        \STATE $E_v^{\rm best} \lA {UB}(E_v(w^*_{\tilde{C}_{t+r}}) | \hat{w}_{\tilde{C}_{t+r}})$
        \STATE $C^{\rm best} \lA \tilde{C}_{t+r}$
      \ENDIF
      \STATE RecursiveCheck($C(L), \tilde{C}_{t+r}, \hat{w}_{C(L)}, \hat{w}_{\tilde{C}_{t+r}}, r$)
      \STATE RecursiveCheck($\tilde{C}_{t+r}, C(R), \hat{w}_{\tilde{C}_{t+r}},\hat{w}_{C(R)}, r$)
    \ENDIF
 \end{algorithmic}
\end{algorithm}

The latter
$\rho > 1$
is used for trick2,
where the next regularization parameter value
is determined in trial-and-error manner.
To formally describe trick2,
let us define a set
$\Gamma$
as a function of $w$
in the following way
\begin{align*}
 \Gamma (w_{\tilde{C}})
 :=
% \bigg\{
 % &
 \Big\{
 \frac{
 \beta(w_{\tilde{C}}, x^\prime_i)
 }{
 \alpha(w_{\tilde{C}}, x^\prime_i) + \delta(g(w_{\tilde{C}}), x^\prime_i)
 } \tilde{C}
 \Big\}_{i \in \cP}
 % \\
 \cup
 % &
 \Big\{
 \frac{
 \alpha(w_{\tilde{C}}, x^\prime_i)
 }{
 \beta(w_{\tilde{C}}, x^\prime_i) + \gamma(g(w_{\tilde{C}}), x^\prime_i)
 } \tilde{C}
 \Big\}_{i \in \cN}.
% \bigg\}
% \cap
% \bigg[\tilde{C}_t, C_u\bigg]
\end{align*}
Then,
our initial trial step is written as
\begin{align}
 \label{eq:nextC.aggr}
 % &
 C^{\rm tmp}
 :=
 % \\
 % \nonumber
 % &
 (\lfloor
 n^\prime
 ({LB}(E_v(w^*_{\tilde{C}_t} ) | \hat{w}_{\tilde{C}_t})
 \! - \!
 E_v^{\rm best}
 \! + \!
 \rho \veps)
 \rfloor
 \! + \!
 1)^{\rm th}(\Gamma (\hat{w}_{\tilde{C}_t}) ),
\end{align}

where
$\rho > 1$
represents how far we step forward.
We then compute an approximate solution
$\hat{w}_{C^{\rm tmp}}$,
and obtain a validation error lower bound
${LB}(E_v(w^*_C) | \hat{w}_{\tilde{C}_t}, \hat{w}_{\tilde{C}^{\rm tmp}})$
by combining
${LB}(E_v(w^*_C) | \hat{w}_{\tilde{C}_t})$
and
${LB}(E_v(w^*_C) | \hat{w}_{C^{\rm tmp}})$.
For accepting this trial step,
we need to make sure that the lower bounds are not smaller
by $\veps$
than the current best
$E_v^{\rm best}$
for any
$C \in [C_t, C^{\rm tmp}]$.
To this end,
we investigate
where the two lower bounds
${LB}(E_v(w^*_C) | \hat{w}_{\tilde{C}_t})$
and
${LB}(E_v(w^*_C) | \hat{w}_{C^{\rm tmp}})$
go below
$E^{\rm best}_v - \veps$.
To formulate this,
let us define the following two functions
\begin{align}
 \label{eq:nextC.left}
 %\hspace*{-9mm}
 C^R(\hat{w}_{C(L)})
 :=
 % &
 (\lfloor
 n^\prime
 ({LB}(E_v( w^*_{C(L)} ) | \hat{w}_{C(L)})
 % \\
 % \nonumber
 % &
 %\hspace*{-9mm}
 ~~\!-\!
 E_v^{\rm best}
 \!+\!
 \veps)
 \rfloor \!+\! 1)^{\rm th}(\Gamma(\hat{w}_{C(L)})),
\end{align}
\begin{align}
 \label{eq:nextC.right}
 %\hspace*{-9mm}
 C^L(\hat{w}_{C(R))})
 :=
 % &
 (\lfloor
 n^\prime
 ({LB}(E_v( w^*_{C(R)} ) | \hat{w}_{C(R)})
 % \\
 % \nonumber
 % &
 %\hspace*{-9mm}
 ~~\!-\!
 E_v^{\rm best}
 \!+\!
 \veps)
 \rfloor \!+\! 1)^{\rm TH}(\Delta(\hat{w}_{C(R)})),
\end{align}
where,
for the latter,
we define
\begin{align*}
 \Delta(w_{\tilde{C}})
 :=
 % &
 \Big\{
 \frac{
 \alpha(w_{\tilde{C}}, x^\prime_i)
 }{
 \beta(w_{\tilde{C}}, x^\prime_i) + \gamma(g(w_{\tilde{C}}), x^\prime_i)
 } \tilde{C}
 \Big\}_{i \in \cP}
 % \\
 \cup
 % &
 \Big\{
 \frac{
 \beta(w_{\tilde{C}}, x^\prime_i)
 }{
 \alpha(w_{\tilde{C}}, x^\prime_i) + \delta(g(w_{\tilde{C}}), x^\prime_i)
 } \tilde{C}
 \Big\}_{i \in \cN},
\end{align*}
and denote the $k^{\rm TH}$-largest element of $\Delta$ as
$k^{\rm TH}(\Delta)$ for any natural number $k$.
The trial step to
$C^{\rm tmp}$
is accepted if
\begin{align*}
 C^L(\hat{w}_{C^{\rm tmp}}) < C^R(\hat{w}_{\tilde{C}_t}).
\end{align*}
If not,
we need to shrink the trial step by using the procedure described in
% Algorithm~\ref{alg:rec.check}.
Algorithm 4.
Briefly speaking,
% Algorithm~\ref{alg:rec.check}
Algorithm 4
conducts a bisection search
until we find two approximate solutions
$\hat{w}_{C(L)}$
and
$\hat{w}_{C(R)}$
that satisfy
$C^L(\hat{w}_{C(L)}) < C^R(\hat{w}_{C(L)})$.
We note that,
with the use of trick2,
the sequence of the regularization parameter values
$\tilde{C}_1, \ldots, \tilde{C}_T$
is not necessarily in increasing order because
they are computed in trial-and-error manner.

\section{Approximate regularization path in terms of validation errors}
\label{app:path}
In this appendix,
we describe
the details of
approximate regularization path
in terms of validation errors
and its experimental results.

By slightly modifying the algorithm,
we can compute
an $\veps$-approximate \emph{regularization path}
whose approximation level is measured
in terms of the validation errors.
Such an $\veps$-approximate regularization path is formulated as a function
\begin{align*}
 W: [C_l, C_u] \to \RR^d, C \mapsto w,
\end{align*}
such that
\begin{align*}
 |E_v(W(C)) - E_v(w^*_C)| \le \veps,
 \forall
 C \in [C_l, C_u].
\end{align*}
In order to compute $W$,
we need an upper bound of the validation errors as well as a lower bound
represented as a function of the regularization parameter.
Given a solution
$\hat{w}_{\tilde{C}}$
for a regularization parameter
$\tilde{C}$,
our basic idea
is to go forward the regularization path
as long as
the difference between the upper
 and the lower bounds are not greater than $\veps$.
We note that,
the approximation quality of our approximate regularization path
is measured in terms of the validation errors,
which is more advantageous
for hyper-parameter tuning tasks
than existing approaches
\cite{Giesen2012b,Giesen2012a,Giesen2014,Mairal2012}
in which
the approximation quality is evaluated in terms of the objective function values.
% As the best of our knowledge,
% all the existing approximate regularization path algorithms
% evaluate
% the approximation quality
% in terms of the objective function values
% \cite{Giesen2012b,Giesen2012a,Giesen2014,Mairal2012},
% which is not suitable for regularization parameter tuning task.
% %

%
%
We compute a validation error upper bound
based on the following simple facts:
\begin{subequations} %08:21�̎��Q
 \label{eq:correctly-classified}
\begin{align}
 \label{eq:correctly-classified.a}
 \hspace*{-2.5mm}
 % \nonumber
 % &
 y^\prime_i = +1
 \text{ and }
 {LB}(w^{*\top}_C x^\prime_i | w_{\tilde{C}}) \ge 0
 % \\
 % & ~~~~~ ~~~~~ ~~~~~ ~~~~~ ~~~~~
 \Rightarrow
 \text{correctly-classified},
 \\
 % \nonumber
 \label{eq:correctly-classified.b}
 \hspace*{-2.5mm}
 % &
 y^\prime_i = -1
 \text{ and }
 {UB}(w^{*\top}_C x^\prime_i | w_{\tilde{C}}) \le 0
 % \\
 % & ~~~~~ ~~~~~ ~~~~~ ~~~~~ ~~~~~
 \Rightarrow
 \text{correctly-classified}.
\end{align}
\end{subequations}
Based on these facts,
we have a lemma for validation error upper bounds similar to
Lemma~\ref{lemm:mis-classified-range}:
\begin{lemm}
 \label{lemm:correctly-classified-range}
 For a validation instance with
 $y^\prime_i = +1$,
 if
 \begin{align*}
  \tilde{C}
  % \le
  <
  C
  \le
  \frac{
  \alpha(\hat{w}_{\tilde{C}}, x^\prime_i)
  }{
  \beta(\hat{w}_{\tilde{C}}, x^\prime_i) + \gamma(g(\hat{w}_{\tilde{C}}), x^\prime_i)
  } \tilde{C}
 % \end{align*}
  ~~~ {\rm or} ~~~
 % \begin{align*}
  \frac{
  \beta(\hat{w}_{\tilde{C}}, x^\prime_i)
  }{
  \alpha(\hat{w}_{\tilde{C}}, x^\prime_i) + \delta(g(\hat{w}_{\tilde{C}}), x^\prime_i)
  } \tilde{C}
  \le
  C
  % \le
  <
  \tilde{C},
 \end{align*}
 then
 the validation instance
 $(x^\prime_i, y^\prime_i)$
 is correctly-classified.
 Similarly,
 for a validation instance with
 $y^\prime_i = -1$,
 if
 \begin{align*}
  \tilde{C}
  % \le
  <
  C
  \le
  \frac{
  \beta(\hat{w}_{\tilde{C}}, x^\prime_i)
  }{
  \alpha(\hat{w}_{\tilde{C}}, x^\prime_i) + \delta(g(\hat{w}_{\tilde{C}}), x^\prime_i)
  } \tilde{C}
 % \end{align*}
 ~~~ { \rm or} ~~~
 % \begin{align*}
  \frac{
  \alpha(\hat{w}_{\tilde{C}}, x^\prime_i)
  }{
  \beta(\hat{w}_{\tilde{C}}, x^\prime_i) + \gamma(g(\hat{w}_{\tilde{C}}), x^\prime_i)
  } \tilde{C}
  \le
  C
  % \le
  <
  \tilde{C},
 \end{align*}
 then
 the validation instance
 $(x^\prime_i, y^\prime_i)$
 is correctly-classified.
\end{lemm}
This lemma can be easily shown by applying
\eq{eq:DS-approx}
to
\eq{eq:correctly-classified}.

Using
Lemma~\ref{lemm:correctly-classified-range},
an upper bound of the validation errors is
represented
as a function of the regularization parameter
$C$
in the following form.
\begin{theo}
 \label{theo:valid-err-UB}
 Using an approximate solution
 $\hat{w}_{\tilde{C}}$
 for a regularization parameter
 $\tilde{C}$,
 the validation error
 $E_v(w^*_C)$
 for any
 $C > 0$ other than $\tilde{C}$
 satisfies
 \begin{align}
  \label{eq:valid-err-UB}
  % &
  E_v(w^*_C)
  &
  \leq
  {UB}(E_v(w^*_C) | \hat{w}_{\tilde{C}})
  % :=
  \\
  \nonumber
  &
  :=
  1-
  \frac{1}{n^\prime}
  \Bigg(
  \sum_{y^\prime_i = +1}
  I \bigg(
  \tilde{C}
  % \le
  <
  C
  \le
  \frac{
  \alpha(\hat{w}_{\tilde{C}}, x^\prime_i)
  }{
  \beta(\hat{w}_{\tilde{C}}, x^\prime_i) + \gamma(g(\hat{w}_{\tilde{C}}), x^\prime_i)
  } \tilde{C}
  \bigg)
  \\
  \nonumber
  &
  ~~~~~ ~~
  +
  \sum_{y^\prime_i = +1}
  I
  \bigg(
  \frac{
  \beta(\hat{w}_{\tilde{C}}, x^\prime_i)
  }{
  \alpha(\hat{w}_{\tilde{C}}, x^\prime_i) + \delta(g(\hat{w}_{\tilde{C}}), x^\prime_i)
  } \tilde{C}
  \le
  C
  % \le
  <
  \tilde{C}
  \bigg)
  \\
  \nonumber
  &
  ~~~~~ ~~
  +
  \sum_{y^\prime_i = -1}
  I
  \bigg(
  \tilde{C}
  % \le
  <
  C
  \le
  \frac{
  \beta(\hat{w}_{\tilde{C}}, x^\prime_i)
  }{
  \alpha(\hat{w}_{\tilde{C}}, x^\prime_i) + \delta(g(\hat{w}_{\tilde{C}}), x^\prime_i)
  } \tilde{C}
  \bigg)
  \\
  \nonumber
  &
  ~~~~~ ~~
  +
  \sum_{y^\prime_i = -1}
  I
  \bigg(
  \frac{
  \alpha(\hat{w}_{\tilde{C}}, x^\prime_i)
  }{
  \beta(\hat{w}_{\tilde{C}}, x^\prime_i) + \gamma(g(\hat{w}_{\tilde{C}}), x^\prime_i)
  } \tilde{C}
  \le
  C
  % \le
  <
  \tilde{C}
  \bigg)
  % \\
  % &
  % ~~~~~ ~~~
  % +
  % \sum_{i \in [n']}
  % I
  % \big\{
  %  \| x^\prime_i \| = 0
  % \big\}
  \Bigg).
  \nonumber
 \end{align}
\end{theo}
Theorem
\ref{theo:valid-err-UB}
is a direct consequence of
Lemma
\ref{lemm:correctly-classified-range}.

\subsection{Algorithm}

% Algorithm~\ref{alg:track.apprx}
Algorithm 5
is the pseudo-code of
our approximate regularization path.
\makeatletter
\renewcommand{\ALG@name}{Algorithm 5}
\renewcommand{\thealgorithm}{}
\makeatother

\begin{algorithm}[t]
 \caption{Tracking an $\veps$-Approximate Regularization Path}
 \label{alg:track.apprx}
 \begin{algorithmic}[1]
   \REQUIRE
   $\{(x_i, y_i)\}_{i \in [n]}$,
   $\{(x^\prime_i, y^\prime_i)\}_{i \in [n^\prime]}$,
   $C_l$,
   $C_u$,
   $\veps$
   \STATE $t \lA 1$, $\tilde{C}_t \lA C_l$
   \WHILE{$\tilde{C}_t \le C_u$}
   \STATE Solve \eq{eq:the-class-of-problems} approximately at $C = \tilde{C}_t$ and obtain $\hat{w}_{\tilde{C}_t}$
   \STATE Set $\tilde{C}_{t+1}$ by \eq{eq:nextC.path}
   \STATE $t \lA t+1$
   \ENDWHILE
   \STATE $T \lA t-1$
   \ENSURE $C_{1}, \ldots , C_{T+1}$, $\hat{w}_{C_1}, \ldots, \hat{w}_{C_T}$
 \end{algorithmic}
\end{algorithm}

\begin{figure}
 % \begin{center}
    \floatbox[{\capbeside\thisfloatsetup{capbesideposition={right,top},capbesidewidth=65mm}}]{figure}[\FBwidth]
  { \hspace*{-7mm}
  \vspace{-1.5em}
  \caption{An illustrative image of tacking the approximate regularization path in terms of validation errors algorithm behavior.
  The blue real lines and red lines represent the validation error lower and upper bounds, respectively.
  The green dashed lines indicate
  the difference between the validation error lower and upper bounds.
  If the green dashed is greater than or equal to $\veps$ ,
  we miss tracking $\veps$ approximation path.
  In such a case,
  the algorithm computes the next approximate solution,
  and update the validation error lower and upper bounds
  based on the new approximate solution.
  }
  }
  { \hspace{5mm}
    \includegraphics[width=0.45\textwidth]{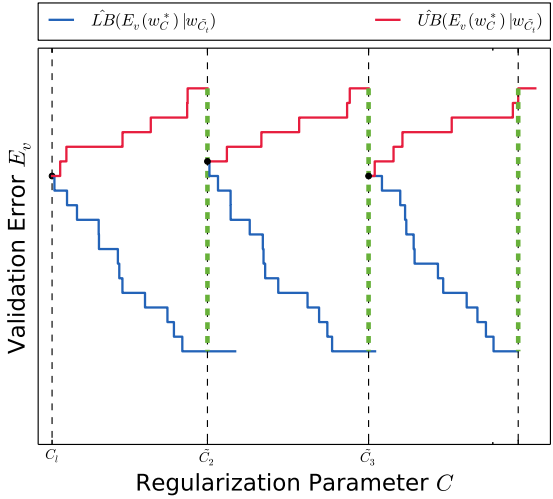}
  }
  \label{fig:ion_path}
 % \end{center}
\end{figure}

The main difference between
% Algorithms~\ref{alg:find.apprx}
Algorithm 2
and
% Algorithm~\ref{alg:track.apprx}
Algorithm 5
is in
how to determine the next regularization parameter value
$\tilde{C}_{t+1}$.
For tracking an approximate solution path,
we need to find
the smallest
$\tilde{C}_{t+1} > C$
such that
the difference between the upper and the lower bounds
$
{UB}(E_v(w^*_{C}) | \hat{w}_{C})
 -
{LB}(E_v(w^*_{C}) | \hat{w}_{C})
$
is greater than or equal to $\veps$.
To formulate this,
let us define
\begin{align*}
 \cP^\prime
 % &
 := \{i \in [n^\prime] |
  y^\prime_i = +1 ,
  % \|x^\prime_i\| > 0,
   {LB}(w^{*\top}_{\tilde{C}_t} x^\prime_i | \hat{w}_{\tilde{C}_t}) \geq 0\},
 % \\
 \cN^\prime
 % &
 := \{i \in [n^\prime] |
 y^\prime_i = -1 ,
  % \|x^\prime_i\| > 0,
  {UB}(w^{*\top}_{\tilde{C}_t} x^\prime_i | \hat{w}_{\tilde{C}_t}) \leq 0\}.
\end{align*}
and
\begin{align*}
 \Lambda (\hat{w}_{\tilde{C}_t})
 :=
% \bigg\{
 &
 \Big\{
 \frac{
 \beta(\hat{w}_{\tilde{C}_t}, x^\prime_i)
 }{
 \alpha(\hat{w}_{\tilde{C}_t}, x^\prime_i) + \delta(g(\hat{w}_{\tilde{C}_t}), x^\prime_i)
 } \tilde{C}_t
 \Big\}_{i \in \cP \cup \cN^\prime}
 % \\
 \cup
 % &
 \Big\{
 \frac{
 \alpha(\hat{w}_{\tilde{C}_t}, x^\prime_i)
 }{
 \beta(\hat{w}_{\tilde{C}_t}, x^\prime_i) + \gamma(g(\hat{w}_{\tilde{C}_t}), x^\prime_i)
 } \tilde{C}_t
 \Big\}_{i \in \cN \cup \cP^\prime},
% \bigg\}
% \cap
% \bigg[\tilde{C}_t, C_u\bigg]
\end{align*}
Then,
$\tilde{C}_{t+1}$
that meets the above requirement is formulated as
\begin{align}
 \label{eq:nextC.path}
 \tilde{C}_{t+1}
 ~\lA~ & (\lfloor
 n^\prime
 ( {LB}(E_v(w^*_{\tilde{C}_t}) | \hat{w}_{\tilde{C}_t})
 % \\ \nonumber
 % &
 % ~~
 - {UB}(E_v(w^*_{\tilde{C}_t}) | \hat{w}_{\tilde{C}_t})
 +
 \veps)
 \rfloor + 1)^{\rm th}(\Lambda).
\end{align}
Figure 5 depicts how $C_{t+1}$ is determined.

Using the output of
% Algorithm~\ref{alg:track.apprx}
Algorithm 5
,
our approximate regularization path is written as
\begin{align*}
W : [C_l, C_u] \rightarrow \mathbb{R}^d,
  \ C \mapsto \sum_{i=1}^{T}
    \bm 1_{[C_i, C_{i+1})}(C) \hat{w}_{C_i},
\end{align*}
where
\begin{align*}
  \bm 1_{[C_i, C_{i+1})}(C) = \begin{cases}
    1 &  \text{if } \ \ C \in  [C_i, C_{i+1}),  \\
    0 &  \text{if } \ \ C \notin [C_i, C_{i+1}).
  \end{cases}
\end{align*}

In approximate regularization path computation,
we need a special treatment
in a pathological situation
that
the signs of the scores of multiple validation instances change at one time
at a regularization parameter value
$C$.
Such a pathological situation is formally stated as follows.
Let
\begin{align*}
 \Omega
 := ~& \{i \in [n^\prime] |  y^\prime_i = +1 ,
  % \|x^\prime_i\| > 0,
   {LB}(w^{*\top}_{\tilde{C}_t} x^\prime_i | \hat{w}_{\tilde{C}_t}) = 0\}
 % \\
 \cup
 % ~
 % &
 \{i \in [n^\prime] | y^\prime_i = -1 ,
   % \|x^\prime_i\| > 0,
   {UB}(w^{*\top}_{\tilde{C}_t} x^\prime_i | \hat{w}_{\tilde{C}_t}) = 0\}.
\end{align*}
Then,
if the size of
$\Omega$
is greater than
\begin{align*}
    (\lfloor
 n^\prime
 ( {LB}(E_v(w^*_{\tilde{C}_t} ) | \hat{w}_{\tilde{C}_t})
 - {UB}(E_v(w^*_{\tilde{C}_t} ) | \hat{w}_{\tilde{C}_t})
 +
 \veps)
 \rfloor + 1),
\end{align*}
% Algorithm~\ref{alg:track.apprx}
Algorithm 5
does not work properly.
Although
such a pathological situation can be considered as an exceptional case
and
treated by tedious book-keeping operations,
in the following experiments,
we simply add an constraint that
$\tilde{C}_{t+1} - \tilde{C}_{t} \ge 10^{-6}$.

\subsection{Experiments}
Here,
we describe the experimental results on approximate regularization path computation.
The experimental setup is same as that in \S\ref{sec:exp}.
Since we cannot use speed-up tricks here,
we have two algorithm options. In the first option ({\bf op4}), we used
optimal solutions $\{w^*_{\tilde{C}_t}\}_{t \in [T]}$ for computing CV error lower bounds.
In the second option ({\bf op5}),
we instead used approximate solutions  $\{\hat{w}_{\tilde{C}_t}\}_{t \in [T]}$.
% we have the following two options:
% \begin{itemize}
%   \item {\bf op4}: Tracking an $\veps$-approximate regularization path by using optimal solutions
%         %$\{w^*_{\tilde{C}_t}\}_{t \in [T]}$
%         %(\S~\ref{subsec:with-optimal})
%   \item {\bf op5}: Tracking an $\veps$-approximate regularization path by using approximate solutions
%         $\{\hat{w}_{\tilde{C}_t}\}_{t \in [T]}$
%         (\S~\ref{subsec:with-approx}).
% \end{itemize}
%
Table~\ref{tab:results.path} shows the experimental results.
Compared with the results in
Table~\ref{tab:results},
we needed to solve more optimization problems
(denoted as $T$)
and hence the total computational cost is larger than simply finding an
$\veps$-approximate regularization parameter.
For large datasets D9 and D10 with $\veps = 0$,
we could not finish the computations
within 100 hours.
\begin{table}[H]
  \begin{center}
   \caption{Complexities and computational costs of approximate regularization path computation experiments.
   For each of the three options
   and
   $\veps \in \{0.10, 0.05, 0.01, 0\}$,
   the number of optimization problems solved (denoted as $T$)
   and the total computational costs (denoted as time)
   are listed.
   Note that,
   for {\bf op5},
   there are no results for $\veps = 0$.
   For D9 and D10 with $\veps = 0$,
   we could not finish the computations
   within 100 hours.
   }
\label{tab:results.path}
   \vspace*{1mm}
   \begin{scriptsize}
    \begin{tabular}{ r| c || r|r || r|r || l || r|r || r|r }
     & & \multicolumn{2}{|c||}{\bf op4} & \multicolumn{2}{|c||}{\bf op5} & & \multicolumn{2}{|c||}{\bf op4} & \multicolumn{2}{|c}{\bf op5} \\
     & & \multicolumn{2}{|c||}{(using $w^*_{\tilde{C}}$)} & \multicolumn{2}{|c||}{(using $\hat{w}_{\tilde{C}}$)} & & \multicolumn{2}{|c||}{(using $w^*_{\tilde{C}}$)} & \multicolumn{2}{|c}{\bf (using $\hat{w}_{\tilde{C}}$)}  \\ \cline{3-6} \cline{8-11}
     \lw{$~~~\veps~~~$}& & \lw{$~~~~T~~~~$} & ~~time~~ & \lw{$~~~~T~~~~$} & ~~time~~ & & \lw{$~~~~T~~~~$} & ~~time~~ & \lw{$~~~~T~~~~$} & ~~time~~ \\
      & & & ~~(sec)~~ & & ~~(sec)~~  & & & ~~(sec)~~ & & ~~(sec)~~  \\ \hline \hline
     0.10 & \multirow{4}{*}{D1}  &    91 &  0.208 &   96 & 0.073 & \multirow{4}{*}{D6} &    238 & 4.828 &  240 & 1.691 \\ \cline{1-1} \cline{3-6} \cline{8-11}
     0.05 &                      &   150 &  0.284 &  180 & 0.118 & &    503 & 9.185 &  507 & 3.518 \\ \cline{1-1} \cline{3-6} \cline{8-11}
     0.01 &                      &   698 &  1.063 & 2095 & 0.597 & &   2332 & 31.17 & 3300 & 17.16 \\ \cline{1-1} \cline{3-6} \cline{8-11}
        0 &                      &  6960 &  7.983 & \multicolumn{2}{|c||}{N.A.} & &  74767 & 836.7 & \multicolumn{2}{|c}{N.A.} \\ \hline \hline

     0.10 & \multirow{4}{*}{D2}  &   504 &  0.367 &  510 & 0.246 & \multirow{4}{*}{D7}  &    732 & 18.56 &  742 & 10.49 \\ \cline{1-1} \cline{3-6} \cline{8-11}
     0.05 &                      &   902 &  0.563 &  982 & 0.444 & &   1316 & 31.77 & 1385 & 18.88 \\ \cline{1-1} \cline{3-6} \cline{8-11}
     0.01 &                      &  4549 &  2.711 & 9404 & 2.365 & &   5820 & 118.4 & 7700 & 76.80 \\ \cline{1-1} \cline{3-6} \cline{8-11}
        0 &                      & 94612 &  68.31 & \multicolumn{2}{|c||}{N.A.} & & 1583578& 43212 & \multicolumn{2}{|c}{N.A.} \\ \hline \hline

     0.10 & \multirow{4}{*}{D3}  &   175 &  1.739 &  186 & 0.592 & \multirow{4}{*}{D8}  &    227 & 1.991 &  229 & 1.410 \\ \cline{1-1} \cline{3-6} \cline{8-11}
     0.05 &                      &   314 &  2.615 &  374 & 1.005 & &    469 & 3.987 &  475 & 2.872 \\ \cline{1-1} \cline{3-6} \cline{8-11}
     0.01 &                      &  1329 &  9.360 & 3248 & 3.409 & &   2382 & 17.95 & 2385 & 14.75 \\ \cline{1-1} \cline{3-6} \cline{8-11}
        0 &                      & 56123 &  292.3 & \multicolumn{2}{|c||}{N.A.} & & 397801 &  5481 &\multicolumn{2}{|c}{N.A.} \\ \hline \hline

     0.10 & \multirow{4}{*}{D4}  &     84 & 0.472 &   86 & 0.201 & \multirow{4}{*}{D9}  &    352 & 844.0 &  357 & 302.6 \\ \cline{1-1} \cline{3-6} \cline{8-11}
     0.05 &                      &    156 & 0.798 &  162 & 0.355 & &    717 &  1209 &  725 & 624.4 \\ \cline{1-1} \cline{3-6} \cline{8-11}
     0.01 &                      &    710 & 2.816 & 1218 & 1.497 & &   3741 &  4985 & 11631& 11185 \\ \cline{1-1} \cline{3-6} \cline{8-11}
        0 &                      &  14833 & 48.06 &\multicolumn{2}{|c||}{N.A.}& & \multicolumn{2}{|c||}{$>$ 100h} & \multicolumn{2}{|c}{N.A.} \\ \hline \hline

     0.10 & \multirow{4}{*}{D5}  &    136 & 0.527 &  138 & 0.185 & \multirow{4}{*}{D10}  &    189 & 145.5 &  200 & 45.18 \\ \cline{1-1} \cline{3-6} \cline{8-11}
     0.05 &                      &    283 & 0.936 &  286 & 0.368 & &    262 & 203.7 &  272 & 61.07 \\ \cline{1-1} \cline{3-6} \cline{8-11}
     0.01 &                      &   1561 & 3.840 & 2306 & 2.086 & &    832 & 524.7 &  851 & 179.7  \\ \cline{1-1} \cline{3-6} \cline{8-11}
        0 &                      &  50101 & 104.9 & \multicolumn{2}{|c||}{N.A.}& &   \multicolumn{2}{|c||}{$>$ 100h} &\multicolumn{2}{|c}{N.A.} % \\ \hline \hline

             % & 0.10 &    238 & 4.828 &  240 & 1.691 \\ \cline{2-6}
             % & 0.05 &    503 & 9.185 &  507 & 3.518 \\ \cline{2-6}
             % & 0.01 &   2332 & 31.17 & 3300 & 17.16 \\ \cline{2-6}
             % &    0 &  74767 & 836.7 & \multicolumn{2}{|c}{N.A.} \\ \hline \hline

             % & 0.10 &    732 & 18.56 &  742 & 10.49 \\ \cline{2-6}
             % & 0.05 &   1316 & 31.77 & 1385 & 18.88 \\ \cline{2-6}
             % & 0.01 &   5820 & 118.4 & 7700 & 76.80 \\ \cline{2-6}
             % &    0 & 1583578& 43212 & \multicolumn{2}{|c}{N.A.} \\ \hline \hline

             % & 0.10 &    227 & 1.991 &  229 & 1.410 \\ \cline{2-6}
             % & 0.05 &    469 & 3.987 &  475 & 2.872 \\ \cline{2-6}
             % & 0.01 &   2382 & 17.95 & 2385 & 14.75 \\ \cline{2-6}
             % &    0 & 397801 &  5481 &\multicolumn{2}{|c}{N.A.} \\ \hline \hline

             % & 0.10 &    352 & 844.0 &  357 & 302.6 \\ \cline{2-6}
             % & 0.01 &    717 &  1209 &  725 & 624.4 \\ \cline{2-6}
             % & 0.05 &   3741 &  4985 & 11631& 11185  \\ \cline{2-6}
             % &    0 & \multicolumn{2}{|c||}{$>$ 100h} & \multicolumn{2}{|c}{N.A.} \\ \hline \hline

             % & 0.10 &    189 & 145.5 &  200 & 45.18 \\ \cline{2-6}
             % & 0.05 &    262 & 203.7 &  272 & 61.07 \\ \cline{2-6}
             % & 0.01 &    832 & 524.7 &  851 & 179.7 \\ \cline{2-6}
             % &    0 &   \multicolumn{2}{|c||}{$>$ 100h} &\multicolumn{2}{|c}{N.A.} \\
    \end{tabular}
   \end{scriptsize}
  \end{center}
\end{table}

\section{Adaptation to cross-validation setup}
\label{app:cv}
All the methods presented above can be straightforwardly adapted to a cross-validation (CV) setup.
Consider
$k$-fold CV
where
$n$
instances
are divided into
$k$
disjoint subsets
$\{F_\kappa\}_{\kappa \in [k]}$
with almost equal size.
Let
$w(\kappa)^*_C$
be the optimal solution
trained without using the instances in
$F_\kappa$.
Then,
the $k$-fold CV error is defined as
\begin{align*}
 E_{k{\rm CV}}(C)
 :=
 \frac{1}{n}
 \sum_{\kappa \in [k]}
 \sum_{i \in F_\kappa}
 I
 \big(
 y_i  w(\kappa)^{*\top}_C x_i < 0
 \big),
\end{align*}
where,
note that,
the CV error is not a function of $w$,
but a function of $C$.
Our algorithm can find
an $\eps$-approximate regularization parameter
at which the
$k$-fold CV error is guaranteed to be no greater by
$\eps$
than the smallest possible $k$-fold CV error.
For each of the $k$ folds,
we can compute a validation error lower bound
as described before.
A lower bound of the entire $k$-fold CV error
can be obtained
by simply summing them up.
%
%\red{
%All the experiments in the next section are conducted in
%cross-validation setup.
%}

%\paragraph{Finding optimal regularization parameter}
%%
%By setting the approximation constant
%$\veps = 0$,
%the proposed algorithm
%can be used to find
%an \emph{optimal} regularization parameter
%at which
%the validation error is exactly minimized
%in the entire range
%$[C_l, C_u]$.
%%
%In this case,
%the solutions
%$\hat{w}_{\tilde{C}_t}$
%must be nearly optimally solved
%in the sense that
%the lower and the upper bounds of the validation errors satisfy
%\begin{align*}
% {LB}(E_v(w^*_{\tilde{C}_t}) | \hat{w}_{\tilde{C}_t})
% =
% {UB}(E_v(w^*_{\tilde{C}_t}) | \hat{w}_{\tilde{C}_t})
% ~
% \forall
% t.
%\end{align*}
%%
%In addition,
%as evident from
%\eq{eq:nextC},
%setting $\veps =0$ makes the step size
%$\tilde{C}_{t+1} - \tilde{C}_t$
%very small,
%which means that many optimization problems must be solved.
%%
%An optimal regularization parameter cannot be found
%in regression problems
%because
%the validation error lower bound of a regression problem
%\emph{cannot}
%be represented as a staircase function
%as depicted in
%\figurename~\ref{fig:}.

\clearpage
%\bibliography{library}
%\bibliographystyle{plain}

\end{document}